\documentclass{article}

\usepackage{PRIMEarxiv}

\usepackage[utf8]{inputenc} 
\usepackage[T1]{fontenc}   
\usepackage{amsmath,amsfonts}
\usepackage{algorithm}
\usepackage{array}
\usepackage[caption=false,font=normalsize,labelfont=sf,textfont=sf]{subfig}
\usepackage{textcomp}
\usepackage{stfloats}
\usepackage{url}
\usepackage{verbatim}
\usepackage{graphicx}
\usepackage{cite}
\usepackage[T1]{fontenc}
\usepackage{graphicx}
\usepackage{amsmath}
\usepackage{xcolor}
\usepackage{amssymb}
\usepackage{amsthm}
\usepackage{float}
\usepackage{hyperref}
\usepackage{enumitem}
\usepackage{booktabs}
\usepackage{amsmath,amssymb,amsfonts}
\usepackage{mathtools}
\usepackage{booktabs}
\usepackage{array}
\usepackage{multirow}
\usepackage{hyperref}
\usepackage{xcolor}
\usepackage{algorithm}
\usepackage{algpseudocode}
\usepackage{float}
\usepackage{caption}
\usepackage{graphicx}
\allowdisplaybreaks     
\newtheorem{theorem}{Theorem}[section]
\usepackage[numbers]{natbib}

\title{Chance-Constrained Bi-Objective Evolutionary Optimization for the Open-Pit Mining Operational Planning Problem}

\author{ \href{https://orcid.org/0000-0002-8043-994X}{\includegraphics[scale=0.06]{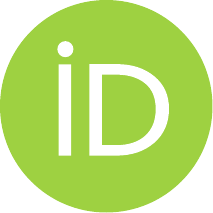}\hspace{1mm}Ishara Hewa Pathiranage} \\
	Machine Learning and Optimisation, \\School of Computer Science and \\Information Technology,\\
  Adelaide University,\\ Adelaide, Australia 
	\And
	\href{https://orcid.org/0000-0002-0036-4782}{\includegraphics[scale=0.06]{orcid.pdf}\hspace{1mm}Aneta Neumann} \\
	Machine Learning and Optimisation, \\
    School of Computer Science and\\ Information Technology,\\
  Adelaide University,\\ 
  Adelaide, Australia \\
}

\begin{document}

\maketitle

\begin{abstract}
Open-pit mining operational planning involves allocating limited resources while satisfying production, equipment, and ore-quality requirements. Existing approaches often assume deterministic ore grades or rely on simulation-based evaluation under uncertainty. In this paper, we propose a chance-constrained bi-objective formulation of the open-pit mining operational planning problem under uncertain ore grades. We aim to maximize ore production and minimize fleet cost while satisfying stochastic quality requirements through chance constraints. We model ore grades as independent normally distributed random variables and derive a deterministic reformulation of the two-sided chance constraints, avoiding sampling or simulation during fitness evaluation. We evaluate four multi-objective evolutionary algorithms on benchmark instances under different levels and structures of uncertainty. Our results show that NSGA-II and NSGA-III generally obtain the best feasible ore-production values, with NSGA-II requiring less computational time. GSEMO has the lowest computational cost but obtains feasible solutions less consistently and generally provides lower solution quality. We also observe that higher confidence and uncertainty levels reduce feasibility, with the effect depending on the uncertainty structure. The results highlight the importance of considering both the magnitude and structure of uncertainty in short-term open-pit operational planning.
\end{abstract}

\keywords{
Open-Pit mining operational planning problem \and Multi-objective evolutionary algorithms \and Uncertainty \and Chance constraints
}

\section{Introduction}

Open-pit mining operations require efficient planning and resource allocation to maximize production while satisfying operational and quality constraints. The open-pit mining operational planning (OPMOP) problem considers short-term decisions, often made for a shift or a day. These decisions include determining extraction rates from ore and waste fronts and assigning loading and hauling equipment, such as shovels and trucks, to these locations. The objective is to meet production and quality requirements while reducing operational costs, particularly those associated with the truck fleet.

Short-term mine planning is important because it directly affects daily production and equipment utilization. It must also account for uncertainties such as variations in ore grades, equipment availability, and traffic conditions. To solve OPMOP, different approaches have been proposed in the past literature, including mathematical programming~\cite{Quigley27052020, Saliba13092020}, heuristic and metaheuristic methods~\cite{10.1007/978-3-319-15892-1_29,SOUZA20101041}, and simulation-based approaches~\cite{GILANI2020101738}. 

Traditional approaches to OPMOP often assume deterministic parameters. However, mining operations are subject to uncertainties in ore grades, material composition, equipment performance, and travel times. Ignoring these uncertainties can result in plans that fail to meet production or quality requirements in practice~\cite{Saliba13092020}.

Chance constraints~\cite{Charnes} provide a framework to account for uncertainty by ensuring that stochastic constraints are satisfied with a specified probability, denoted by $\alpha$. Chance-constrained
approaches have been applied to mining ~\cite{Mohtasham03042021,10.1145/3449639.3459382} and logistics~\cite{FARINA201653,10.5555/2888116.2888219}. However, large-scale chance-constrained optimization problems are difficult to solve using exact methods due to their combinatorial and nonlinear characteristics~\cite{10254112}. Evolutionary algorithms (EAs) provide an alternative and have been applied to a range of mining optimization problems~\cite{ELSAYED202077,GILANI2020101738, KHAN2018428,10612008,9504901,DynamicOPMSP}. Among them, multi-objective evolutionary algorithms (MOEAs) are effective when multiple conflicting objectives need to be considered simultaneously~\cite{ijcai2022p665,10.1145/3377930.3390162}.

In this paper, we propose a bi-objective evolutionary formulation for the chance-constrained OPMOP that simultaneously maximizes ore production and minimizes fleet cost while satisfying operational and chance-constrained ore-quality requirements. Ore-grade uncertainty is incorporated through chance constraints with confidence level $\alpha$, which controls the risk of violating the specified quality bounds. This formulation allows production, fleet cost, operational requirements, and ore-grade uncertainty to be considered within the same optimization framework.

\subsection{Related Work}

The OPMOP problem has been studied for equipment allocation, scheduling, and routing decisions in open-pit mining. Souza et al.~\cite{SOUZA20101041} proposed a hybrid heuristic combining greedy randomized adaptive search procedures and general variable neighborhood search for large-scale instances. This work was later extended to a
multi-objective formulation in~\cite{coelho2012}, allowing production and cost objectives to be considered simultaneously.

Multi-objective evolutionary algorithms have also been studied for OPMOP. Alexandre et al.~\cite{10.1145/3067695.3076004,10.1007/978-3-319-15892-1_29} compared multi-objective evolutionary algorithms for operational planning in open-pit mining. This work was further extended to truck dispatch problems, where evolutionary approaches were used to obtain trade-offs between operational objectives~\cite{alexandre2019multi}.

Several studies have considered uncertainty in open-pit mining optimization. Stochastic optimization models have been developed to account for variability in ore grades, supply, and operational conditions~\cite{Saliba13092020}. These approaches account for uncertainty explicitly but often consider a single objective. Chance constraints~\cite{Charnes} have also been used in mining optimization. Mohtasham et al.~\cite{Mohtasham03042021} proposed a chance-constrained goal programming model for truck-shovel allocation under uncertainty in loader production.

Evolutionary algorithms have also been applied to chance-constrained mining problems. A chance-constrained stockpile blending problem was studied using a differential evolution based algorithm in~\cite{10.1145/3449639.3459382}. A single-objective evolutionary approach for a chance-constrained open-pit mine scheduling problem was considered in~\cite{10254112}, where uncertainty was represented using multiple geological realizations. In~\cite{DynamicOPMSP}, a dynamic chance-constrained open-pit mine scheduling problem was studied with stochastic block economic values and time-varying mining and processing capacities, and four MOEAs were evaluated.

However, limited work has considered chance-constrained short-term operational planning with uncertain ore-quality attributes. In particular, the combination of multi-objective evolutionary optimization and a deterministic reformulation of two-sided chance-constrained quality requirements has not been studied for OPMOP.

\subsection{Contributions}
\label{sec:contributions}

To the best of our knowledge, this work is among the first to study OPMOP under uncertain ore-quality attributes using a bi-objective evolutionary formulation with chance-constrained quality requirements. We formulate the problem to maximize ore production and minimize fleet cost, model ore grades as independent normally distributed random variables, and derive sufficient deterministic conditions for the two-sided chance constraints, enabling evaluation without sampling or simulation. We develop a problem-specific evolutionary framework and evaluate GSEMO, NSGA-II, NSGA-III, and MOEA/D under different confidence levels, uncertainty magnitudes, and relationships between expected ore grades and their uncertainty.

The remainder of the paper is organized as follows. Section~\ref{sec:model} presents the problem formulation and the deterministic reformulation of the chance constraints. Section~\ref{sec:methodology} describes the solution representation, operators, and algorithms. Section~\ref{sec:experiments} presents the experimental setup and results. Finally, Section~\ref{sec:conclusion} concludes the paper and discusses future work.

\section{Open-Pit Mining Operational Planning Problem under Uncertainty}
\label{sec:model}

We consider a short-term open-pit operational planning problem over a fixed planning horizon of length $H$. The objective is to determine the number of haulage trips between mining fronts and destinations using a heterogeneous truck fleet while satisfying production, equipment, and quality constraints.

Let $O$ and $W$ denote the sets of ore and waste fronts, respectively, and let $F=O\cup W$ be the set of all mining fronts. Let $C$, $D$, $S$, $K$, and $Q$ denote the sets of crushers, waste dumps, shovels, truck types, and quality attributes, respectively.

For each ore front $o\in O$ and waste front $w\in W$, let $M_o$ and $M_w$ denote the available material. Let $L_f\in\mathbb{Z}_{>0}$ be the maximum number of shovels assigned to front $f\in F$, and let $a_{fs}\in\{0,1\}$ indicate whether shovel $s\in S$ is compatible with front $f$. Each shovel $s$ has loading rate $\overline{S}_s>0$, while truck type $k\in K$ has capacity $\mathrm{Cap}_k>0$, availability $N_k\in\mathbb{Z}_{>0}$, and cost coefficient $\lambda_k>0$.

We assume that the grade of quality attribute $q\in Q$ at ore front $o\in O$ follows $\widetilde{q}_{oq}\sim\mathcal{N}(\mu_{oq},\sigma_{oq}^{2})$, where $\mu_{oq}$ and $\sigma_{oq}^{2}$ are its mean and variance. For each quality attribute, the ore-grade random variables are assumed independent across ore fronts. For crusher $c\in C$, $\underline{q}_{cq}$ and $\overline{q}_{cq}$ denote the lower and upper quality bounds.

Let $\widehat{P}_c$ and $\widehat{R}_d$ denote the minimum material delivery requirements at crusher $c$ and waste dump $d$, respectively, and let $\tau_{ock}$ and $\tau_{wdk}$ denote the corresponding ore and waste cycle times. The decision variables are $n_{ock}\in\mathbb{Z}_{\geq0}$ and $n_{wdk}\in\mathbb{Z}_{\geq0}$, denoting the numbers of ore and waste haulage trips, respectively, and $y_{fs}\in\{0,1\}$, indicating whether shovel $s$ is assigned to front $f$.

A solution is represented by $x=(n_{ock},n_{wdk},y_{fs})$. We define the hauled ore and waste tonnage as $h_{oc}=\sum_{k\in K}\mathrm{Cap}_k n_{ock}$ and $h_{wd}=\sum_{k\in K}\mathrm{Cap}_k n_{wdk}$, respectively.
The total material extracted from each ore and waste front is $p_o=\sum_{c\in C}h_{oc}$ and $r_w=\sum_{d\in D}h_{wd}$, while the total material delivered to each crusher and waste dump is $P_c=\sum_{o\in O}h_{oc}$ and $R_d=\sum_{w\in W}h_{wd}$.

For crusher $c$ and quality attribute $q$, define $B_{cq}=\sum_{o\in O}\widetilde{q}_{oq}h_{oc}$, so that
$B_{cq}/P_c$ is the corresponding blended ore grade. The total operating time required from truck type $k$ is
$T_k=\sum_{o\in O}\sum_{c\in C}\tau_{ock}n_{ock} +\sum_{w\in W}\sum_{d\in D}\tau_{wdk}n_{wdk}$.
The number of active trucks is $u_k=\min\{\lceil T_k/H\rceil,N_k\}$. The total ore production and fleet cost are $P(x)=\sum_{o\in O}p_o$ and $C_{\mathrm{fleet}}(x)=\sum_{k\in K}\lambda_k u_k$.

The bi-objective problem is
\begin{align}
\max \quad & P(x) \label{eq:obj_prod}\\
\min \quad & C_{\mathrm{fleet}}(x) \label{eq:obj_fleet}\\[0.5em]
\text{s.t.}\quad
& p_o\leq M_o
&& \forall o\in O, \label{eq:reserve_ore}\\
& r_w\leq M_w
&& \forall w\in W, \label{eq:reserve_waste}\\
& \sum_{s\in S}y_{fs}\leq L_f
&& \forall f\in F, \label{eq:shovel_limit}\\
& \sum_{f\in F}y_{fs}\leq 1
&& \forall s\in S, \label{eq:one_front}\\
& y_{fs}\leq a_{fs}
&& \forall f\in F,\;s\in S, \label{eq:compatibility}\\
& p_o\leq H\sum_{s\in S}\overline{S}_s y_{os}
&& \forall o\in O, \label{eq:prod_ore}\\
& r_w\leq H\sum_{s\in S}\overline{S}_s y_{ws}
&& \forall w\in W, \label{eq:prod_waste}\\
& P_c\geq\widehat{P}_c
&& \forall c\in C, \label{eq:crusher_req}\\
& R_d\geq\widehat{R}_d
&& \forall d\in D, \label{eq:dump_req}\\
& \Pr\!\left(
\underline{q}_{cq}P_c
\leq B_{cq}
\leq
\overline{q}_{cq}P_c
\right)\geq\alpha
&& \forall c\in C,\;q\in Q,
\label{eq:chance_constraint}\\
& T_k\leq HN_k
&& \forall k\in K.
\label{eq:truck_time}
\end{align}

The two objectives represent the trade-off between ore production and fleet cost. Constraints~\eqref{eq:reserve_ore}-\eqref{eq:truck_time} enforce reserve, shovel-assignment, loading-capacity, minimum-delivery, ore-quality, and truck-availability requirements. In particular,~\eqref{eq:chance_constraint} requires the blended ore quality to remain within its specified bounds with probability at least $\alpha$.

\subsection{Deterministic Reformulation}

Since $B_{cq}$ is a linear combination of independent normally distributed ore-grade variables, it is also normally distributed. Its mean and standard deviation are
\[
m_{cq}=\mathbb{E}[B_{cq}]
=\sum_{o\in O}\mu_{oq}h_{oc},\]
\[
s_{cq}=\sqrt{\mathrm{Var}[B_{cq}]}
=\sqrt{\sum_{o\in O}\sigma_{oq}^{2}h_{oc}^{2}}.
\]

The following theorem provides sufficient deterministic conditions for the two-sided chance constraint.

\begin{theorem}
\label{thm:deterministic_reformulation}
Assume that the ore-grade random variables $\widetilde{q}_{oq}$ are independent and normally distributed across ore fronts. Sufficient deterministic conditions for satisfying Eq.~\eqref{eq:chance_constraint} are
\begin{align}
m_{cq}-z_\beta s_{cq}
&\geq \underline{q}_{cq}P_c,
\label{eq:soc_lower}\\
m_{cq}+z_\beta s_{cq}
&\leq \overline{q}_{cq}P_c,
\label{eq:soc_upper}
\end{align}
where $\beta=(1+\alpha)/2$ and $z_\beta=\Phi^{-1}(\beta)$.
\end{theorem}

\begin{proof}
Let
\[
L=\{B_{cq}\geq\underline{q}_{cq}P_c\},
\qquad
U=\{B_{cq}\leq\overline{q}_{cq}P_c\},
\]
denote the events that the lower and upper quality bounds are satisfied, respectively. The two-sided chance constraint requires $\Pr(L\cap U)\geq\alpha$.

By Bonferroni's inequality~\cite{KucukyavuzJiang2022},
\[
\Pr(L\cap U)\geq\Pr(L)+\Pr(U)-1.
\]
If both $\Pr(L)$ and $\Pr(U)$ are required to be at least $\beta$, then Bonferroni's inequality gives
\[
\Pr(L\cap U)\geq \beta+\beta-1=2\beta-1.
\]
To guarantee the original requirement $\Pr(L\cap U)\geq\alpha$, it is therefore sufficient to choose $\beta$ such that 
\[
2\beta-1=\alpha,
\]
which gives $\beta=\frac{1+\alpha}{2}$.

Since $B_{cq}\sim\mathcal{N}(m_{cq},s_{cq}^{2})$, the requirement $\Pr(L)\geq\beta$ is equivalent to $m_{cq}-z_\beta s_{cq}\geq\underline{q}_{cq}P_c$, which gives Eq.~\eqref{eq:soc_lower}. Similarly, $\Pr(U)\geq\beta$ is equivalent to $m_{cq}+z_\beta s_{cq}\leq\overline{q}_{cq}P_c$, which gives Eq.~\eqref{eq:soc_upper}. Therefore, Eqs.~\eqref{eq:soc_lower}-\eqref{eq:soc_upper} imply $\Pr(L\cap U)\geq\alpha$ and are therefore sufficient deterministic conditions for the original two-sided chance constraint.
\end{proof}

\subsection{Bi-Objective Fitness Function}

The fitness of a solution $x$ is defined as
\begin{align}
f_1(x)
&=
\begin{cases}
P(x), & V(x)=0,\\
P(x)-\rho V(x), & \text{otherwise},
\end{cases}
\label{eq:objective_1_opm}\\
f_2(x)
&=
\begin{cases}
C_{\mathrm{fleet}}(x), & V(x)=0,\\
C_{\mathrm{fleet}}(x)+\rho V(x), & \text{otherwise},
\end{cases}
\label{eq:objective_2_opm}
\end{align}

The total constraint violation is defined as $V(x)=V_{\mathrm{op}}(x)+V_{\mathrm{cc}}(x)$, where $V_{\mathrm{op}}(x)$ and $V_{\mathrm{cc}}(x)$ represent violations of the operational constraints and deterministic chance-constraint conditions, respectively. The violation of each inequality is measured by the amount by which it exceeds its feasible bound, and these violations are summed. The penalty factor is scaled according to the minimum throughput requirements of the instance as $\rho = \max\left(\sum_{c\in C}\widehat{P}_c + \sum_{d\in D}\widehat{R}_d, 1.0 \right)$.

\section{Multi-Objective Evolutionary Algorithms for OPMOP}
\label{sec:methodology}

We evaluate the proposed chance-constrained bi-objective OPMOP using
GSEMO~\cite{gsemo}, NSGA-II~\cite{Deb2002AFA}, NSGA-III~\cite{6600851}, and MOEA/D~\cite{4358754}, representing different dominance- and decomposition-based search mechanisms. These algorithms have demonstrated effective performance on combinatorial and constrained multi-objective optimization problems~\cite{10.1145/3638529.3654066,10.1162/evco_a_00360}. GSEMO maintains a variable-size set of non-dominated solutions, NSGA-II uses non-dominated sorting and crowding distance, NSGA-III uses reference-direction-based selection, and MOEA/D decomposes the problem into scalar subproblems. For MOEA/D, we use the Tchebycheff aggregation method. All algorithms use the same problem-specific representation, initialization, mutation, and repair mechanisms, while NSGA-II, NSGA-III, and MOEA/D additionally use crossover.

\subsection{Initial Solution Generation}

We generate the initial solutions using a constructive heuristic. We first assign waste movements and then generate ore-production decisions. For each selected front, we assign an eligible shovel and generate trip variables while considering the operational constraints, including loading capacity, reserve limits, truck-time availability, and variable bounds.

For ore fronts, we use a quality score based on the distance between the expected ore grades and the crusher quality bounds relative to grade variability. This guides the initialization toward ore fronts with favorable quality characteristics. We do not explicitly enforce the chance constraints during initialization, and any violations are handled through the penalty mechanism during fitness evaluation.

\subsection{Variation Operators}

We use problem-specific crossover and mutation operators that maintain the block structure of the haulage decisions.

For crossover, we use a uniform block crossover with probability $p_c$. For each ore front-crusher pair, we exchange the complete block of trip variables for all truck types between two offspring with probability $0.5$. We apply the same procedure to each waste front-dump pair and exchange the shovel-assignment block for each processed front independently with probability $0.5$. We then apply mutation and repair to the resulting offspring.

For mutation, we use four mechanisms. First, we select each trip-count variable independently with probability $p_m$ and perturb it by a value sampled uniformly from $\{-5,-4,-3,-2,-1,1,2,3,4,5\}$, while restricting the resulting value to its bounds. Second, we apply bit-flip mutation to each shovel-assignment variable with probability $p_s$.

Third, with probability $p_r$, we relocate an active ore or waste trip to a randomly selected front-destination pair of the same material type without changing the truck type. Finally, with probability $p_t$, we replace the truck type of an active front-destination assignment and adjust the number of trips to maintain a similar transported tonnage.

\begin{table}[!t]
\centering
\caption{Characteristics of the instances~\cite{SOUZA20101041}.}
\label{tab:opm_instances}
\resizebox{\textwidth}{!}{
\begin{tabular}{l c c l c c l}
\toprule
Instance & \# Fronts & \# Shovels & Shovel Characteristics & \# Attributes & \# Trucks & Truck Characteristics \\
\midrule
opm1 & 17 & 8 & (4, 900\,\text{t}), (2, 1000\,\text{t}), (2, 1100\,\text{t}) & 10 & 30 & (15, 50\,\text{t}), (15, 80\,\text{t}) \\
opm2 & 17 & 8 & (4, 900\,\text{t}), (2, 1000\,\text{t}), (2, 1100\,\text{t}) & 10 & 30 & (15, 50\,\text{t}), (15, 80\,\text{t}) \\
opm5 & 17 & 8 & (4, 900\,\text{t}), (2, 1000\,\text{t}), (2, 1100\,\text{t}) & 5 & 30 & (15, 50\,\text{t}), (15, 80\,\text{t}) \\
opm6 & 17 & 8 & (4, 900\,\text{t}), (2, 1000\,\text{t}), (2, 1100\,\text{t}) & 5 & 30 & (15, 50\,\text{t}), (15, 80\,\text{t}) \\
\bottomrule
\end{tabular}
}
\end{table}

\begin{figure}[!t]
    \centering
    \includegraphics[width=0.95\textwidth]{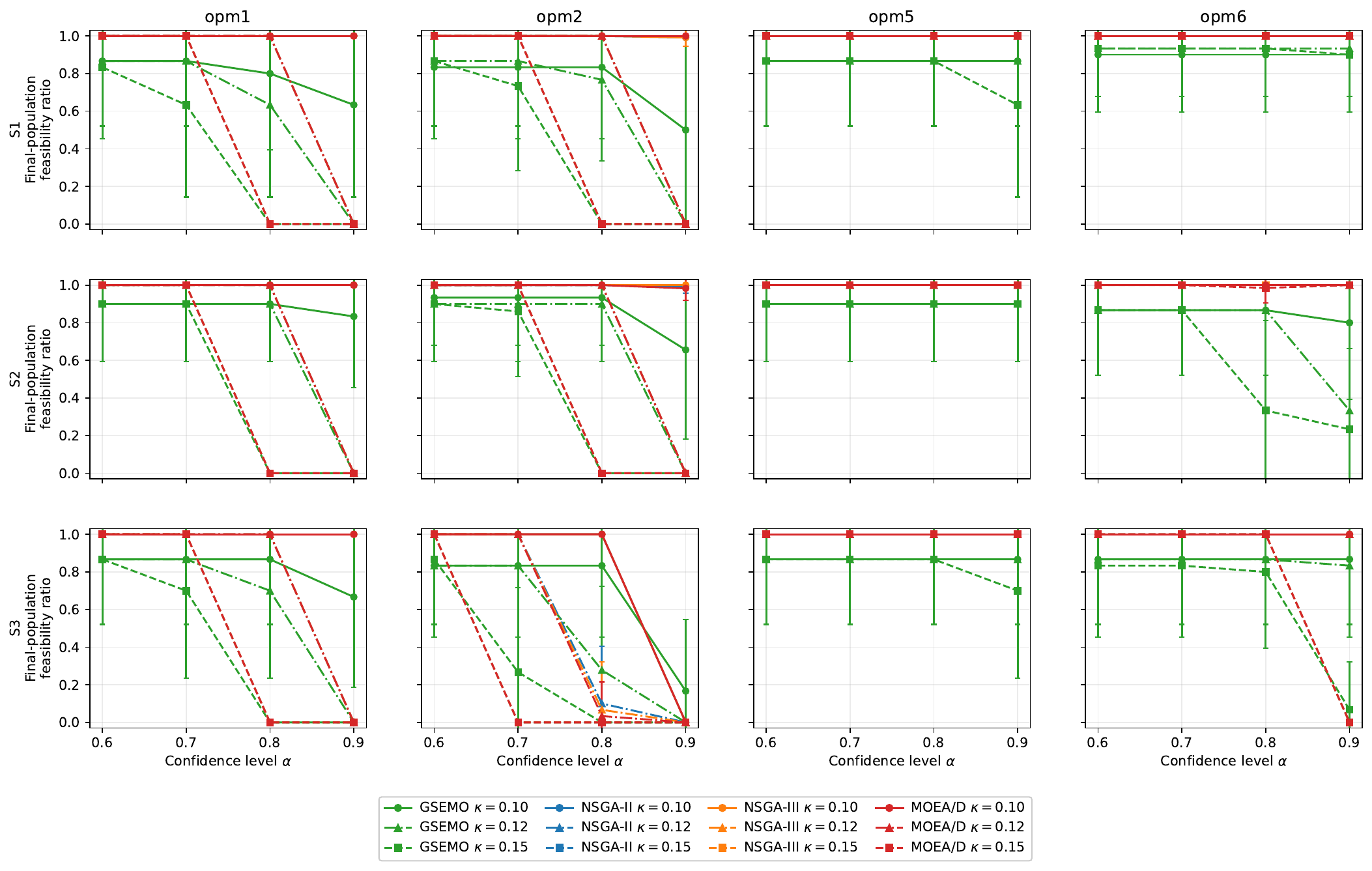}
    \caption{Feasibility ratio across confidence levels $\alpha$ for all algorithms under uncertainty settings S1, S2, and S3 (top to bottom), for different values of $\kappa$.}
    \label{fig:feasibility_ratio}
\end{figure}

\begin{figure}[!t]
    \centering
    \includegraphics[width=0.95\textwidth]{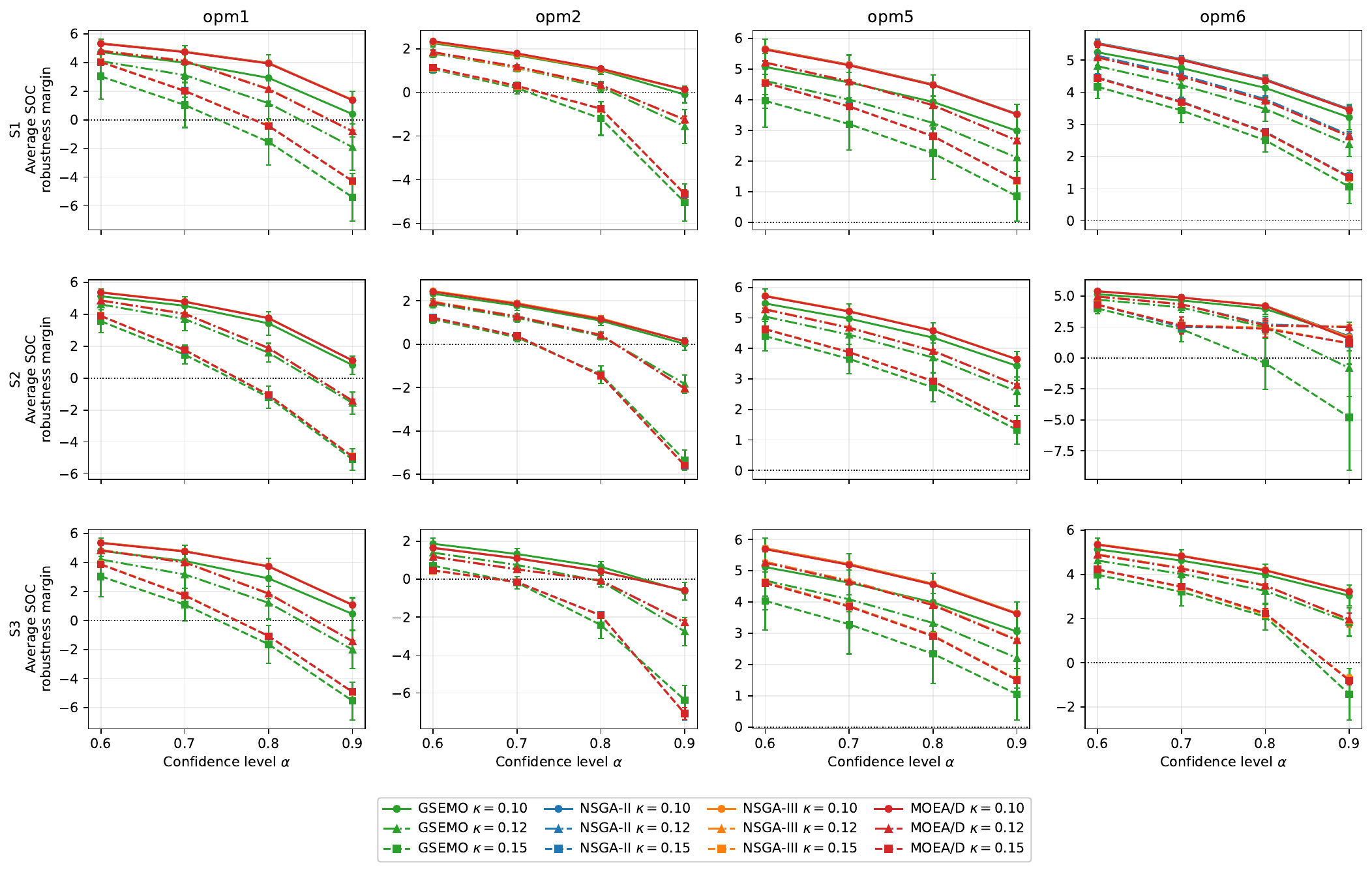}
    \caption{Average chance-constraint margin across confidence levels $\alpha$ under uncertainty settings S1, S2, and S3 (top to bottom) for different values of $\kappa$.}
    \label{fig:robustness_margin}
\end{figure}

\subsection{Repair Mechanism}

After variation, we apply a repair operator to restore feasibility with respect to the structural and operational constraints. The repair procedure first adjusts shovel assignments to satisfy compatibility, front-specific assignment limits, and assignment uniqueness. Trips associated with inactive fronts are removed. The trip variables are then adjusted, when necessary, to satisfy reserve limits, shovel-loading capacities, and truck-time availability. The chance constraints are not explicitly repaired. Their violations are included in the total constraint violation $V(x)$ and handled by the penalty mechanism during fitness evaluation. 

\begin{table}[!t]
\centering
\caption{Mean, standard deviation, statistical comparisons, and feasible-run count for the best feasible ore production under S1 across all $\kappa$ values.}
\label{tab:f1_setting_S1}
\resizebox{\textwidth}{!}{
\begin{tabular}{l c c rrrr rrrr rrrr rrrr}
\toprule
Instance & $\kappa$ & $\alpha$ & \multicolumn{4}{c}{GSEMO\textsuperscript{(1)}} & \multicolumn{4}{c}{NSGA-II\textsuperscript{(2)}} & \multicolumn{4}{c}{NSGA-III\textsuperscript{(3)}} & \multicolumn{4}{c}{MOEA/D\textsuperscript{(4)}} \\
\cmidrule(lr){4-7} \cmidrule(lr){8-11} \cmidrule(lr){12-15} \cmidrule(lr){16-19}
 & & & Mean & Std & Feas. & Stat & Mean & Std & Feas. & Stat & Mean & Std & Feas. & Stat & Mean & Std & Feas. & Stat \\
\midrule
opm1 & 0.10 & 0.60 & 5876.9 & 81.5 & 26/30 & 2\textsuperscript{(-)}3\textsuperscript{(-)}4\textsuperscript{(*)} & \textbf{5976.3} & 10.0 & 30/30 & 1\textsuperscript{(+)}3\textsuperscript{(*)}4\textsuperscript{(+)} & 5973.7 & 14.5 & 30/30 & 1\textsuperscript{(+)}2\textsuperscript{(*)}4\textsuperscript{(+)} & 5941.0 & 30.2 & 30/30 & 1\textsuperscript{(*)}2\textsuperscript{(-)}3\textsuperscript{(-)} \\
 &  & 0.70 & 5876.9 & 81.5 & 26/30 & 2\textsuperscript{(-)}3\textsuperscript{(-)}4\textsuperscript{(*)} & \textbf{5976.3} & 10.0 & 30/30 & 1\textsuperscript{(+)}3\textsuperscript{(*)}4\textsuperscript{(+)} & 5973.7 & 14.5 & 30/30 & 1\textsuperscript{(+)}2\textsuperscript{(*)}4\textsuperscript{(+)} & 5941.0 & 30.2 & 30/30 & 1\textsuperscript{(*)}2\textsuperscript{(-)}3\textsuperscript{(-)} \\
 &  & 0.80 & 5883.3 & 81.6 & 24/30 & 2\textsuperscript{(-)}3\textsuperscript{(-)}4\textsuperscript{(*)} & \textbf{5976.3} & 10.0 & 30/30 & 1\textsuperscript{(+)}3\textsuperscript{(*)}4\textsuperscript{(+)} & 5973.7 & 14.5 & 30/30 & 1\textsuperscript{(+)}2\textsuperscript{(*)}4\textsuperscript{(+)} & 5941.0 & 30.2 & 30/30 & 1\textsuperscript{(*)}2\textsuperscript{(-)}3\textsuperscript{(-)} \\
 &  & 0.90 & 5884.2 & 83.4 & 19/30 & 2\textsuperscript{(-)}3\textsuperscript{(-)}4\textsuperscript{(*)} & \textbf{5977.0} & 7.9 & 30/30 & 1\textsuperscript{(+)}3\textsuperscript{(*)}4\textsuperscript{(+)} & 5973.0 & 14.7 & 30/30 & 1\textsuperscript{(+)}2\textsuperscript{(*)}4\textsuperscript{(+)} & 5941.0 & 30.2 & 30/30 & 1\textsuperscript{(*)}2\textsuperscript{(-)}3\textsuperscript{(-)} \\
\cmidrule(lr){2-19}
 & 0.12 & 0.60 & 5876.9 & 81.5 & 26/30 & 2\textsuperscript{(-)}3\textsuperscript{(-)}4\textsuperscript{(*)} & \textbf{5976.3} & 10.0 & 30/30 & 1\textsuperscript{(+)}3\textsuperscript{(*)}4\textsuperscript{(+)} & 5973.7 & 14.5 & 30/30 & 1\textsuperscript{(+)}2\textsuperscript{(*)}4\textsuperscript{(+)} & 5941.0 & 30.2 & 30/30 & 1\textsuperscript{(*)}2\textsuperscript{(-)}3\textsuperscript{(-)} \\
 &  & 0.70 & 5876.9 & 81.5 & 26/30 & 2\textsuperscript{(-)}3\textsuperscript{(-)}4\textsuperscript{(*)} & \textbf{5976.3} & 10.0 & 30/30 & 1\textsuperscript{(+)}3\textsuperscript{(*)}4\textsuperscript{(+)} & 5973.7 & 14.5 & 30/30 & 1\textsuperscript{(+)}2\textsuperscript{(*)}4\textsuperscript{(+)} & 5941.0 & 30.2 & 30/30 & 1\textsuperscript{(*)}2\textsuperscript{(-)}3\textsuperscript{(-)} \\
 &  & 0.80 & 5884.2 & 83.4 & 19/30 & 2\textsuperscript{(-)}3\textsuperscript{(-)}4\textsuperscript{(*)} & \textbf{5977.3} & 8.7 & 30/30 & 1\textsuperscript{(+)}3\textsuperscript{(*)}4\textsuperscript{(+)} & 5976.3 & 10.0 & 30/30 & 1\textsuperscript{(+)}2\textsuperscript{(*)}4\textsuperscript{(+)} & 5941.0 & 30.2 & 30/30 & 1\textsuperscript{(*)}2\textsuperscript{(-)}3\textsuperscript{(-)} \\
 &  & 0.90 & - & - & 0/30 &  & - & - & 0/30 &  & - & - & 0/30 &  & - & - & 0/30 &  \\
\cmidrule(lr){2-19}
 & 0.15 & 0.60 & 5880.0 & 81.6 & 25/30 & 2\textsuperscript{(-)}3\textsuperscript{(-)}4\textsuperscript{(*)} & \textbf{5976.3} & 10.0 & 30/30 & 1\textsuperscript{(+)}3\textsuperscript{(*)}4\textsuperscript{(+)} & 5973.7 & 14.5 & 30/30 & 1\textsuperscript{(+)}2\textsuperscript{(*)}4\textsuperscript{(+)} & 5941.0 & 30.2 & 30/30 & 1\textsuperscript{(*)}2\textsuperscript{(-)}3\textsuperscript{(-)} \\
 &  & 0.70 & 5884.2 & 83.4 & 19/30 & 2\textsuperscript{(-)}3\textsuperscript{(-)}4\textsuperscript{(*)} & \textbf{5977.3} & 8.7 & 30/30 & 1\textsuperscript{(+)}3\textsuperscript{(*)}4\textsuperscript{(+)} & 5976.3 & 10.0 & 30/30 & 1\textsuperscript{(+)}2\textsuperscript{(*)}4\textsuperscript{(+)} & 5941.0 & 30.2 & 30/30 & 1\textsuperscript{(*)}2\textsuperscript{(-)}3\textsuperscript{(-)} \\
 &  & 0.80 & - & - & 0/30 &  & - & - & 0/30 &  & - & - & 0/30 &  & - & - & 0/30 &  \\
 &  & 0.90 & - & - & 0/30 &  & - & - & 0/30 &  & - & - & 0/30 &  & - & - & 0/30 &  \\
\midrule
opm2 & 0.10 & 0.60 & 5878.8 & 79.9 & 25/30 & 2\textsuperscript{(-)}3\textsuperscript{(-)}4\textsuperscript{(*)} & \textbf{5978.7} & 5.1 & 30/30 & 1\textsuperscript{(+)}3\textsuperscript{(*)}4\textsuperscript{(+)} & 5976.7 & 10.6 & 30/30 & 1\textsuperscript{(+)}2\textsuperscript{(*)}4\textsuperscript{(+)} & 5949.3 & 30.0 & 30/30 & 1\textsuperscript{(*)}2\textsuperscript{(-)}3\textsuperscript{(-)} \\
 &  & 0.70 & 5878.8 & 79.9 & 25/30 & 2\textsuperscript{(-)}3\textsuperscript{(-)}4\textsuperscript{(*)} & \textbf{5978.7} & 5.1 & 30/30 & 1\textsuperscript{(+)}3\textsuperscript{(*)}4\textsuperscript{(+)} & 5977.3 & 10.1 & 30/30 & 1\textsuperscript{(+)}2\textsuperscript{(*)}4\textsuperscript{(+)} & 5949.3 & 30.0 & 30/30 & 1\textsuperscript{(*)}2\textsuperscript{(-)}3\textsuperscript{(-)} \\
 &  & 0.80 & 5879.2 & 80.5 & 25/30 & 2\textsuperscript{(-)}3\textsuperscript{(-)}4\textsuperscript{(*)} & 5976.7 & 10.6 & 30/30 & 1\textsuperscript{(+)}3\textsuperscript{(*)}4\textsuperscript{(+)} & \textbf{5977.3} & 8.7 & 30/30 & 1\textsuperscript{(+)}2\textsuperscript{(*)}4\textsuperscript{(+)} & 5949.3 & 30.0 & 30/30 & 1\textsuperscript{(*)}2\textsuperscript{(-)}3\textsuperscript{(-)} \\
 &  & 0.90 & 5898.7 & 73.9 & 15/30 & 2\textsuperscript{(-)}3\textsuperscript{(-)}4\textsuperscript{(*)} & \textbf{5974.0} & 15.0 & 30/30 & 1\textsuperscript{(+)}3\textsuperscript{(*)}4\textsuperscript{(+)} & \textbf{5974.0} & 15.0 & 30/30 & 1\textsuperscript{(+)}2\textsuperscript{(*)}4\textsuperscript{(+)} & 5934.7 & 30.3 & 30/30 & 1\textsuperscript{(*)}2\textsuperscript{(-)}3\textsuperscript{(-)} \\
\cmidrule(lr){2-19}
 & 0.12 & 0.60 & 5876.2 & 80.4 & 26/30 & 2\textsuperscript{(-)}3\textsuperscript{(-)}4\textsuperscript{(*)} & \textbf{5975.3} & 13.6 & 30/30 & 1\textsuperscript{(+)}3\textsuperscript{(*)}4\textsuperscript{(+)} & 5974.7 & 12.8 & 30/30 & 1\textsuperscript{(+)}2\textsuperscript{(*)}4\textsuperscript{(+)} & 5940.3 & 34.3 & 30/30 & 1\textsuperscript{(*)}2\textsuperscript{(-)}3\textsuperscript{(-)} \\
 &  & 0.70 & 5876.2 & 80.4 & 26/30 & 2\textsuperscript{(-)}3\textsuperscript{(-)}4\textsuperscript{(*)} & \textbf{5978.0} & 8.1 & 30/30 & 1\textsuperscript{(+)}3\textsuperscript{(*)}4\textsuperscript{(+)} & 5975.3 & 13.6 & 30/30 & 1\textsuperscript{(+)}2\textsuperscript{(*)}4\textsuperscript{(+)} & 5940.3 & 34.3 & 30/30 & 1\textsuperscript{(*)}2\textsuperscript{(-)}3\textsuperscript{(-)} \\
 &  & 0.80 & 5884.8 & 78.5 & 23/30 & 2\textsuperscript{(-)}3\textsuperscript{(-)}4\textsuperscript{(*)} & \textbf{5978.0} & 8.1 & 30/30 & 1\textsuperscript{(+)}3\textsuperscript{(*)}4\textsuperscript{(+)} & 5974.7 & 13.8 & 30/30 & 1\textsuperscript{(+)}2\textsuperscript{(*)}4\textsuperscript{(+)} & 5937.7 & 34.2 & 30/30 & 1\textsuperscript{(*)}2\textsuperscript{(-)}3\textsuperscript{(-)} \\
 &  & 0.90 & - & - & 0/30 &  & - & - & 0/30 &  & - & - & 0/30 &  & - & - & 0/30 &  \\
\cmidrule(lr){2-19}
 & 0.15 & 0.60 & 5875.8 & 79.8 & 26/30 & 2\textsuperscript{(-)}3\textsuperscript{(-)}4\textsuperscript{(*)} & \textbf{5978.0} & 8.1 & 30/30 & 1\textsuperscript{(+)}3\textsuperscript{(*)}4\textsuperscript{(+)} & 5975.3 & 13.6 & 30/30 & 1\textsuperscript{(+)}2\textsuperscript{(*)}4\textsuperscript{(+)} & 5945.0 & 32.5 & 30/30 & 1\textsuperscript{(*)}2\textsuperscript{(-)}3\textsuperscript{(-)} \\
 &  & 0.70 & 5889.5 & 79.3 & 22/30 & 2\textsuperscript{(-)}3\textsuperscript{(-)}4\textsuperscript{(*)} & \textbf{5977.3} & 8.7 & 30/30 & 1\textsuperscript{(+)}3\textsuperscript{(*)}4\textsuperscript{(+)} & 5976.0 & 13.3 & 30/30 & 1\textsuperscript{(+)}2\textsuperscript{(*)}4\textsuperscript{(+)} & 5945.0 & 32.5 & 30/30 & 1\textsuperscript{(*)}2\textsuperscript{(-)}3\textsuperscript{(-)} \\
 &  & 0.80 & - & - & 0/30 &  & - & - & 0/30 &  & - & - & 0/30 &  & - & - & 0/30 &  \\
 &  & 0.90 & - & - & 0/30 &  & - & - & 0/30 &  & - & - & 0/30 &  & - & - & 0/30 &  \\
\midrule
opm5 & 0.10 & 0.60 & 5879.6 & 78.4 & 26/30 & 2\textsuperscript{(-)}3\textsuperscript{(-)}4\textsuperscript{(*)} & \textbf{5978.3} & 6.5 & 30/30 & 1\textsuperscript{(+)}3\textsuperscript{(*)}4\textsuperscript{(+)} & 5978.0 & 6.1 & 30/30 & 1\textsuperscript{(+)}2\textsuperscript{(*)}4\textsuperscript{(+)} & 5936.3 & 29.3 & 30/30 & 1\textsuperscript{(*)}2\textsuperscript{(-)}3\textsuperscript{(-)} \\
 &  & 0.70 & 5879.6 & 78.4 & 26/30 & 2\textsuperscript{(-)}3\textsuperscript{(-)}4\textsuperscript{(*)} & \textbf{5978.3} & 6.5 & 30/30 & 1\textsuperscript{(+)}3\textsuperscript{(*)}4\textsuperscript{(+)} & 5978.0 & 6.1 & 30/30 & 1\textsuperscript{(+)}2\textsuperscript{(*)}4\textsuperscript{(+)} & 5936.3 & 29.3 & 30/30 & 1\textsuperscript{(*)}2\textsuperscript{(-)}3\textsuperscript{(-)} \\
 &  & 0.80 & 5879.6 & 78.4 & 26/30 & 2\textsuperscript{(-)}3\textsuperscript{(-)}4\textsuperscript{(*)} & \textbf{5978.3} & 6.5 & 30/30 & 1\textsuperscript{(+)}3\textsuperscript{(*)}4\textsuperscript{(+)} & 5978.0 & 6.1 & 30/30 & 1\textsuperscript{(+)}2\textsuperscript{(*)}4\textsuperscript{(+)} & 5936.3 & 29.3 & 30/30 & 1\textsuperscript{(*)}2\textsuperscript{(-)}3\textsuperscript{(-)} \\
 &  & 0.90 & 5879.6 & 78.4 & 26/30 & 2\textsuperscript{(-)}3\textsuperscript{(-)}4\textsuperscript{(*)} & \textbf{5978.3} & 6.5 & 30/30 & 1\textsuperscript{(+)}3\textsuperscript{(*)}4\textsuperscript{(+)} & 5978.0 & 6.1 & 30/30 & 1\textsuperscript{(+)}2\textsuperscript{(*)}4\textsuperscript{(+)} & 5936.3 & 29.3 & 30/30 & 1\textsuperscript{(*)}2\textsuperscript{(-)}3\textsuperscript{(-)} \\
\cmidrule(lr){2-19}
 & 0.12 & 0.60 & 5879.6 & 78.4 & 26/30 & 2\textsuperscript{(-)}3\textsuperscript{(-)}4\textsuperscript{(*)} & \textbf{5977.0} & 9.5 & 30/30 & 1\textsuperscript{(+)}3\textsuperscript{(*)}4\textsuperscript{(+)} & 5975.0 & 14.1 & 30/30 & 1\textsuperscript{(+)}2\textsuperscript{(*)}4\textsuperscript{(+)} & 5945.0 & 30.7 & 30/30 & 1\textsuperscript{(*)}2\textsuperscript{(-)}3\textsuperscript{(-)} \\
 &  & 0.70 & 5879.6 & 78.4 & 26/30 & 2\textsuperscript{(-)}3\textsuperscript{(-)}4\textsuperscript{(*)} & \textbf{5977.0} & 9.5 & 30/30 & 1\textsuperscript{(+)}3\textsuperscript{(*)}4\textsuperscript{(+)} & 5975.0 & 14.1 & 30/30 & 1\textsuperscript{(+)}2\textsuperscript{(*)}4\textsuperscript{(+)} & 5945.0 & 30.7 & 30/30 & 1\textsuperscript{(*)}2\textsuperscript{(-)}3\textsuperscript{(-)} \\
 &  & 0.80 & 5879.6 & 78.4 & 26/30 & 2\textsuperscript{(-)}3\textsuperscript{(-)}4\textsuperscript{(*)} & \textbf{5977.0} & 9.5 & 30/30 & 1\textsuperscript{(+)}3\textsuperscript{(*)}4\textsuperscript{(+)} & 5975.0 & 14.1 & 30/30 & 1\textsuperscript{(+)}2\textsuperscript{(*)}4\textsuperscript{(+)} & 5945.0 & 30.7 & 30/30 & 1\textsuperscript{(*)}2\textsuperscript{(-)}3\textsuperscript{(-)} \\
 &  & 0.90 & 5879.6 & 78.4 & 26/30 & 2\textsuperscript{(-)}3\textsuperscript{(-)}4\textsuperscript{(*)} & \textbf{5977.0} & 9.5 & 30/30 & 1\textsuperscript{(+)}3\textsuperscript{(*)}4\textsuperscript{(+)} & 5975.0 & 14.1 & 30/30 & 1\textsuperscript{(+)}2\textsuperscript{(*)}4\textsuperscript{(+)} & 5945.0 & 30.7 & 30/30 & 1\textsuperscript{(*)}2\textsuperscript{(-)}3\textsuperscript{(-)} \\
\cmidrule(lr){2-19}
 & 0.15 & 0.60 & 5875.8 & 79.8 & 26/30 & 2\textsuperscript{(-)}3\textsuperscript{(-)}4\textsuperscript{(*)} & \textbf{5977.0} & 9.5 & 30/30 & 1\textsuperscript{(+)}3\textsuperscript{(*)}4\textsuperscript{(+)} & 5976.3 & 10.0 & 30/30 & 1\textsuperscript{(+)}2\textsuperscript{(*)}4\textsuperscript{(+)} & 5945.0 & 31.6 & 30/30 & 1\textsuperscript{(*)}2\textsuperscript{(-)}3\textsuperscript{(-)} \\
 &  & 0.70 & 5875.8 & 79.8 & 26/30 & 2\textsuperscript{(-)}3\textsuperscript{(-)}4\textsuperscript{(*)} & \textbf{5977.0} & 9.5 & 30/30 & 1\textsuperscript{(+)}3\textsuperscript{(*)}4\textsuperscript{(+)} & 5976.3 & 10.0 & 30/30 & 1\textsuperscript{(+)}2\textsuperscript{(*)}4\textsuperscript{(+)} & 5945.0 & 31.6 & 30/30 & 1\textsuperscript{(*)}2\textsuperscript{(-)}3\textsuperscript{(-)} \\
 &  & 0.80 & 5875.8 & 79.8 & 26/30 & 2\textsuperscript{(-)}3\textsuperscript{(-)}4\textsuperscript{(*)} & \textbf{5977.0} & 9.5 & 30/30 & 1\textsuperscript{(+)}3\textsuperscript{(*)}4\textsuperscript{(+)} & 5976.3 & 10.0 & 30/30 & 1\textsuperscript{(+)}2\textsuperscript{(*)}4\textsuperscript{(+)} & 5945.0 & 31.6 & 30/30 & 1\textsuperscript{(*)}2\textsuperscript{(-)}3\textsuperscript{(-)} \\
 &  & 0.90 & 5882.6 & 81.2 & 19/30 & 2\textsuperscript{(-)}3\textsuperscript{(-)}4\textsuperscript{(*)} & \textbf{5977.3} & 8.7 & 30/30 & 1\textsuperscript{(+)}3\textsuperscript{(*)}4\textsuperscript{(+)} & 5976.3 & 8.5 & 30/30 & 1\textsuperscript{(+)}2\textsuperscript{(*)}4\textsuperscript{(+)} & 5945.0 & 31.6 & 30/30 & 1\textsuperscript{(*)}2\textsuperscript{(-)}3\textsuperscript{(-)} \\
\midrule
opm6 & 0.10 & 0.60 & 5884.4 & 75.9 & 27/30 & 2\textsuperscript{(-)}3\textsuperscript{(-)}4\textsuperscript{(*)} & \textbf{5977.0} & 9.5 & 30/30 & 1\textsuperscript{(+)}3\textsuperscript{(*)}4\textsuperscript{(+)} & 5976.3 & 10.0 & 30/30 & 1\textsuperscript{(+)}2\textsuperscript{(*)}4\textsuperscript{(+)} & 5948.0 & 28.0 & 30/30 & 1\textsuperscript{(*)}2\textsuperscript{(-)}3\textsuperscript{(-)} \\
 &  & 0.70 & 5884.4 & 75.9 & 27/30 & 2\textsuperscript{(-)}3\textsuperscript{(-)}4\textsuperscript{(*)} & \textbf{5977.0} & 9.5 & 30/30 & 1\textsuperscript{(+)}3\textsuperscript{(*)}4\textsuperscript{(+)} & 5976.3 & 10.0 & 30/30 & 1\textsuperscript{(+)}2\textsuperscript{(*)}4\textsuperscript{(+)} & 5948.0 & 28.0 & 30/30 & 1\textsuperscript{(*)}2\textsuperscript{(-)}3\textsuperscript{(-)} \\
 &  & 0.80 & 5884.4 & 75.9 & 27/30 & 2\textsuperscript{(-)}3\textsuperscript{(-)}4\textsuperscript{(*)} & \textbf{5977.0} & 9.5 & 30/30 & 1\textsuperscript{(+)}3\textsuperscript{(*)}4\textsuperscript{(+)} & 5976.3 & 10.0 & 30/30 & 1\textsuperscript{(+)}2\textsuperscript{(*)}4\textsuperscript{(+)} & 5948.0 & 28.0 & 30/30 & 1\textsuperscript{(*)}2\textsuperscript{(-)}3\textsuperscript{(-)} \\
 &  & 0.90 & 5884.4 & 75.9 & 27/30 & 2\textsuperscript{(-)}3\textsuperscript{(-)}4\textsuperscript{(*)} & \textbf{5977.0} & 9.5 & 30/30 & 1\textsuperscript{(+)}3\textsuperscript{(*)}4\textsuperscript{(+)} & 5976.3 & 10.0 & 30/30 & 1\textsuperscript{(+)}2\textsuperscript{(*)}4\textsuperscript{(+)} & 5948.0 & 28.0 & 30/30 & 1\textsuperscript{(*)}2\textsuperscript{(-)}3\textsuperscript{(-)} \\
\cmidrule(lr){2-19}
 & 0.12 & 0.60 & 5878.2 & 78.1 & 28/30 & 2\textsuperscript{(-)}3\textsuperscript{(-)}4\textsuperscript{(*)} & \textbf{5978.0} & 8.1 & 30/30 & 1\textsuperscript{(+)}3\textsuperscript{(*)}4\textsuperscript{(+)} & 5975.0 & 13.1 & 30/30 & 1\textsuperscript{(+)}2\textsuperscript{(*)}4\textsuperscript{(+)} & 5942.3 & 33.0 & 30/30 & 1\textsuperscript{(*)}2\textsuperscript{(-)}3\textsuperscript{(-)} \\
 &  & 0.70 & 5878.2 & 78.1 & 28/30 & 2\textsuperscript{(-)}3\textsuperscript{(-)}4\textsuperscript{(*)} & \textbf{5978.0} & 8.1 & 30/30 & 1\textsuperscript{(+)}3\textsuperscript{(*)}4\textsuperscript{(+)} & 5975.0 & 13.1 & 30/30 & 1\textsuperscript{(+)}2\textsuperscript{(*)}4\textsuperscript{(+)} & 5942.3 & 33.0 & 30/30 & 1\textsuperscript{(*)}2\textsuperscript{(-)}3\textsuperscript{(-)} \\
 &  & 0.80 & 5878.2 & 78.1 & 28/30 & 2\textsuperscript{(-)}3\textsuperscript{(-)}4\textsuperscript{(*)} & \textbf{5978.0} & 8.1 & 30/30 & 1\textsuperscript{(+)}3\textsuperscript{(*)}4\textsuperscript{(+)} & 5975.0 & 13.1 & 30/30 & 1\textsuperscript{(+)}2\textsuperscript{(*)}4\textsuperscript{(+)} & 5942.3 & 33.0 & 30/30 & 1\textsuperscript{(*)}2\textsuperscript{(-)}3\textsuperscript{(-)} \\
 &  & 0.90 & 5878.2 & 78.1 & 28/30 & 2\textsuperscript{(-)}3\textsuperscript{(-)}4\textsuperscript{(*)} & \textbf{5978.0} & 8.1 & 30/30 & 1\textsuperscript{(+)}3\textsuperscript{(*)}4\textsuperscript{(+)} & 5975.0 & 13.1 & 30/30 & 1\textsuperscript{(+)}2\textsuperscript{(*)}4\textsuperscript{(+)} & 5942.3 & 33.0 & 30/30 & 1\textsuperscript{(*)}2\textsuperscript{(-)}3\textsuperscript{(-)} \\
\cmidrule(lr){2-19}
 & 0.15 & 0.60 & 5878.2 & 78.1 & 28/30 & 2\textsuperscript{(-)}3\textsuperscript{(-)}4\textsuperscript{(*)} & \textbf{5977.3} & 11.4 & 30/30 & 1\textsuperscript{(+)}3\textsuperscript{(*)}4\textsuperscript{(+)} & 5976.7 & 13.0 & 30/30 & 1\textsuperscript{(+)}2\textsuperscript{(*)}4\textsuperscript{(+)} & 5947.0 & 28.4 & 30/30 & 1\textsuperscript{(*)}2\textsuperscript{(-)}3\textsuperscript{(-)} \\
 &  & 0.70 & 5878.2 & 78.1 & 28/30 & 2\textsuperscript{(-)}3\textsuperscript{(-)}4\textsuperscript{(*)} & \textbf{5977.3} & 11.4 & 30/30 & 1\textsuperscript{(+)}3\textsuperscript{(*)}4\textsuperscript{(+)} & 5976.7 & 13.0 & 30/30 & 1\textsuperscript{(+)}2\textsuperscript{(*)}4\textsuperscript{(+)} & 5947.0 & 28.4 & 30/30 & 1\textsuperscript{(*)}2\textsuperscript{(-)}3\textsuperscript{(-)} \\
 &  & 0.80 & 5878.2 & 78.1 & 28/30 & 2\textsuperscript{(-)}3\textsuperscript{(-)}4\textsuperscript{(*)} & \textbf{5977.3} & 11.4 & 30/30 & 1\textsuperscript{(+)}3\textsuperscript{(*)}4\textsuperscript{(+)} & 5976.7 & 13.0 & 30/30 & 1\textsuperscript{(+)}2\textsuperscript{(*)}4\textsuperscript{(+)} & 5947.0 & 28.4 & 30/30 & 1\textsuperscript{(*)}2\textsuperscript{(-)}3\textsuperscript{(-)} \\
 &  & 0.90 & 5877.4 & 79.5 & 27/30 & 2\textsuperscript{(-)}3\textsuperscript{(-)}4\textsuperscript{(*)} & 5976.0 & 13.3 & 30/30 & 1\textsuperscript{(+)}3\textsuperscript{(*)}4\textsuperscript{(+)} & \textbf{5977.3} & 11.4 & 30/30 & 1\textsuperscript{(+)}2\textsuperscript{(*)}4\textsuperscript{(+)} & 5947.0 & 28.4 & 30/30 & 1\textsuperscript{(*)}2\textsuperscript{(-)}3\textsuperscript{(-)} \\
\bottomrule
\end{tabular}
}
\end{table}

\begin{table}[!t]
\centering
\caption{Mean, standard deviation, statistical comparisons, and feasible-run count for the best feasible ore production under S2 across all $\kappa$ values.}
\label{tab:f1_setting_S2}
\resizebox{\textwidth}{!}{
\begin{tabular}{l c c rrrr rrrr rrrr rrrr}
\toprule
Instance & $\kappa$ & $\alpha$ & \multicolumn{4}{c}{GSEMO\textsuperscript{(1)}} & \multicolumn{4}{c}{NSGA-II\textsuperscript{(2)}} & \multicolumn{4}{c}{NSGA-III\textsuperscript{(3)}} & \multicolumn{4}{c}{MOEA/D\textsuperscript{(4)}} \\
\cmidrule(lr){4-7} \cmidrule(lr){8-11} \cmidrule(lr){12-15} \cmidrule(lr){16-19}
 & & & Mean & Std & Feas. & Stat & Mean & Std & Feas. & Stat & Mean & Std & Feas. & Stat & Mean & Std & Feas. & Stat \\
\midrule
opm1 & 0.10 & 0.60 & 5877.0 & 78.9 & 27/30 & 2\textsuperscript{(-)}3\textsuperscript{(-)}4\textsuperscript{(*)} & \textbf{5979.3} & 11.1 & 30/30 & 1\textsuperscript{(+)}3\textsuperscript{(*)}4\textsuperscript{(+)} & \textbf{5979.3} & 12.3 & 30/30 & 1\textsuperscript{(+)}2\textsuperscript{(*)}4\textsuperscript{(+)} & 5947.7 & 34.8 & 30/30 & 1\textsuperscript{(*)}2\textsuperscript{(-)}3\textsuperscript{(-)} \\
 &  & 0.70 & 5877.0 & 78.9 & 27/30 & 2\textsuperscript{(-)}3\textsuperscript{(-)}4\textsuperscript{(*)} & \textbf{5979.3} & 11.1 & 30/30 & 1\textsuperscript{(+)}3\textsuperscript{(*)}4\textsuperscript{(+)} & \textbf{5979.3} & 12.3 & 30/30 & 1\textsuperscript{(+)}2\textsuperscript{(*)}4\textsuperscript{(+)} & 5947.7 & 34.8 & 30/30 & 1\textsuperscript{(*)}2\textsuperscript{(-)}3\textsuperscript{(-)} \\
 &  & 0.80 & 5877.0 & 78.9 & 27/30 & 2\textsuperscript{(-)}3\textsuperscript{(-)}4\textsuperscript{(*)} & \textbf{5979.3} & 11.1 & 30/30 & 1\textsuperscript{(+)}3\textsuperscript{(*)}4\textsuperscript{(+)} & \textbf{5979.3} & 12.3 & 30/30 & 1\textsuperscript{(+)}2\textsuperscript{(*)}4\textsuperscript{(+)} & 5947.7 & 34.8 & 30/30 & 1\textsuperscript{(*)}2\textsuperscript{(-)}3\textsuperscript{(-)} \\
 &  & 0.90 & 5883.2 & 78.8 & 25/30 & 2\textsuperscript{(-)}3\textsuperscript{(-)}4\textsuperscript{(*)} & \textbf{5980.0} & 9.1 & 30/30 & 1\textsuperscript{(+)}3\textsuperscript{(*)}4\textsuperscript{(+)} & 5979.3 & 11.1 & 30/30 & 1\textsuperscript{(+)}2\textsuperscript{(*)}4\textsuperscript{(+)} & 5947.7 & 34.8 & 30/30 & 1\textsuperscript{(*)}2\textsuperscript{(-)}3\textsuperscript{(-)} \\
\cmidrule(lr){2-19}
 & 0.12 & 0.60 & 5877.0 & 78.9 & 27/30 & 2\textsuperscript{(-)}3\textsuperscript{(-)}4\textsuperscript{(*)} & \textbf{5979.3} & 11.1 & 30/30 & 1\textsuperscript{(+)}3\textsuperscript{(*)}4\textsuperscript{(+)} & \textbf{5979.3} & 12.3 & 30/30 & 1\textsuperscript{(+)}2\textsuperscript{(*)}4\textsuperscript{(+)} & 5947.7 & 34.8 & 30/30 & 1\textsuperscript{(*)}2\textsuperscript{(-)}3\textsuperscript{(-)} \\
 &  & 0.70 & 5877.0 & 78.9 & 27/30 & 2\textsuperscript{(-)}3\textsuperscript{(-)}4\textsuperscript{(*)} & \textbf{5979.3} & 11.1 & 30/30 & 1\textsuperscript{(+)}3\textsuperscript{(*)}4\textsuperscript{(+)} & \textbf{5979.3} & 12.3 & 30/30 & 1\textsuperscript{(+)}2\textsuperscript{(*)}4\textsuperscript{(+)} & 5947.7 & 34.8 & 30/30 & 1\textsuperscript{(*)}2\textsuperscript{(-)}3\textsuperscript{(-)} \\
 &  & 0.80 & 5877.0 & 78.9 & 27/30 & 2\textsuperscript{(-)}3\textsuperscript{(-)}4\textsuperscript{(*)} & \textbf{5980.0} & 10.5 & 30/30 & 1\textsuperscript{(+)}3\textsuperscript{(*)}4\textsuperscript{(+)} & 5978.7 & 11.7 & 30/30 & 1\textsuperscript{(+)}2\textsuperscript{(*)}4\textsuperscript{(+)} & 5947.7 & 34.8 & 30/30 & 1\textsuperscript{(*)}2\textsuperscript{(-)}3\textsuperscript{(-)} \\
 &  & 0.90 & - & - & 0/30 &  & - & - & 0/30 &  & - & - & 0/30 &  & - & - & 0/30 &  \\
\cmidrule(lr){2-19}
 & 0.15 & 0.60 & 5877.0 & 78.9 & 27/30 & 2\textsuperscript{(-)}3\textsuperscript{(-)}4\textsuperscript{(*)} & \textbf{5979.3} & 11.1 & 30/30 & 1\textsuperscript{(+)}3\textsuperscript{(*)}4\textsuperscript{(+)} & \textbf{5979.3} & 12.3 & 30/30 & 1\textsuperscript{(+)}2\textsuperscript{(*)}4\textsuperscript{(+)} & 5947.7 & 34.8 & 30/30 & 1\textsuperscript{(*)}2\textsuperscript{(-)}3\textsuperscript{(-)} \\
 &  & 0.70 & 5877.0 & 78.9 & 27/30 & 2\textsuperscript{(-)}3\textsuperscript{(-)}4\textsuperscript{(*)} & \textbf{5980.0} & 10.5 & 30/30 & 1\textsuperscript{(+)}3\textsuperscript{(*)}4\textsuperscript{(+)} & 5978.7 & 11.7 & 30/30 & 1\textsuperscript{(+)}2\textsuperscript{(*)}4\textsuperscript{(+)} & 5947.7 & 34.8 & 30/30 & 1\textsuperscript{(*)}2\textsuperscript{(-)}3\textsuperscript{(-)} \\
 &  & 0.80 & - & - & 0/30 &  & - & - & 0/30 &  & - & - & 0/30 &  & - & - & 0/30 &  \\
 &  & 0.90 & - & - & 0/30 &  & - & - & 0/30 &  & - & - & 0/30 &  & - & - & 0/30 &  \\
\midrule
opm2 & 0.10 & 0.60 & 5870.7 & 79.9 & 28/30 & 2\textsuperscript{(-)}3\textsuperscript{(-)}4\textsuperscript{(*)} & \textbf{5975.7} & 13.8 & 30/30 & 1\textsuperscript{(+)}3\textsuperscript{(*)}4\textsuperscript{(+)} & 5975.0 & 14.1 & 30/30 & 1\textsuperscript{(+)}2\textsuperscript{(*)}4\textsuperscript{(+)} & 5949.3 & 28.8 & 30/30 & 1\textsuperscript{(*)}2\textsuperscript{(-)}3\textsuperscript{(-)} \\
 &  & 0.70 & 5870.4 & 79.3 & 28/30 & 2\textsuperscript{(-)}3\textsuperscript{(-)}4\textsuperscript{(*)} & 5975.0 & 14.1 & 30/30 & 1\textsuperscript{(+)}3\textsuperscript{(*)}4\textsuperscript{(+)} & \textbf{5976.3} & 10.0 & 30/30 & 1\textsuperscript{(+)}2\textsuperscript{(*)}4\textsuperscript{(+)} & 5949.3 & 28.8 & 30/30 & 1\textsuperscript{(*)}2\textsuperscript{(-)}3\textsuperscript{(-)} \\
 &  & 0.80 & 5870.7 & 79.9 & 28/30 & 2\textsuperscript{(-)}3\textsuperscript{(-)}4\textsuperscript{(*)} & \textbf{5975.7} & 13.8 & 30/30 & 1\textsuperscript{(+)}3\textsuperscript{(*)}4\textsuperscript{(+)} & \textbf{5975.7} & 12.8 & 30/30 & 1\textsuperscript{(+)}2\textsuperscript{(*)}4\textsuperscript{(+)} & 5950.0 & 28.3 & 30/30 & 1\textsuperscript{(*)}2\textsuperscript{(-)}3\textsuperscript{(-)} \\
 &  & 0.90 & 5872.5 & 77.7 & 20/30 & 2\textsuperscript{(-)}3\textsuperscript{(-)}4\textsuperscript{(*)} & \textbf{5976.3} & 12.5 & 30/30 & 1\textsuperscript{(+)}3\textsuperscript{(*)}4\textsuperscript{(+)} & 5975.7 & 13.8 & 30/30 & 1\textsuperscript{(+)}2\textsuperscript{(*)}4\textsuperscript{(+)} & 5949.3 & 27.3 & 30/30 & 1\textsuperscript{(*)}2\textsuperscript{(-)}3\textsuperscript{(-)} \\
\cmidrule(lr){2-19}
 & 0.12 & 0.60 & 5873.3 & 80.1 & 27/30 & 2\textsuperscript{(-)}3\textsuperscript{(-)}4\textsuperscript{(*)} & \textbf{5978.3} & 6.5 & 30/30 & 1\textsuperscript{(+)}3\textsuperscript{(*)}4\textsuperscript{(+)} & 5977.0 & 10.9 & 30/30 & 1\textsuperscript{(+)}2\textsuperscript{(*)}4\textsuperscript{(+)} & 5951.3 & 27.3 & 30/30 & 1\textsuperscript{(*)}2\textsuperscript{(-)}3\textsuperscript{(-)} \\
 &  & 0.70 & 5873.3 & 80.1 & 27/30 & 2\textsuperscript{(-)}3\textsuperscript{(-)}4\textsuperscript{(*)} & 5977.0 & 12.1 & 30/30 & 1\textsuperscript{(+)}3\textsuperscript{(*)}4\textsuperscript{(+)} & \textbf{5977.7} & 9.0 & 30/30 & 1\textsuperscript{(+)}2\textsuperscript{(*)}4\textsuperscript{(+)} & 5950.0 & 27.8 & 30/30 & 1\textsuperscript{(*)}2\textsuperscript{(-)}3\textsuperscript{(-)} \\
 &  & 0.80 & 5873.0 & 79.6 & 27/30 & 2\textsuperscript{(-)}3\textsuperscript{(-)}4\textsuperscript{(*)} & \textbf{5977.0} & 10.9 & 30/30 & 1\textsuperscript{(+)}3\textsuperscript{(*)}4\textsuperscript{(+)} & 5975.0 & 15.0 & 30/30 & 1\textsuperscript{(+)}2\textsuperscript{(*)}4\textsuperscript{(+)} & 5950.0 & 27.8 & 30/30 & 1\textsuperscript{(*)}2\textsuperscript{(-)}3\textsuperscript{(-)} \\
 &  & 0.90 & - & - & 0/30 &  & - & - & 0/30 &  & - & - & 0/30 &  & - & - & 0/30 &  \\
\cmidrule(lr){2-19}
 & 0.15 & 0.60 & 5873.3 & 80.1 & 27/30 & 2\textsuperscript{(-)}3\textsuperscript{(-)}4\textsuperscript{(*)} & \textbf{5977.0} & 9.5 & 30/30 & 1\textsuperscript{(+)}3\textsuperscript{(*)}4\textsuperscript{(+)} & \textbf{5977.0} & 12.1 & 30/30 & 1\textsuperscript{(+)}2\textsuperscript{(*)}4\textsuperscript{(+)} & 5949.3 & 26.8 & 30/30 & 1\textsuperscript{(*)}2\textsuperscript{(-)}3\textsuperscript{(-)} \\
 &  & 0.70 & 5873.1 & 78.8 & 26/30 & 2\textsuperscript{(-)}3\textsuperscript{(-)}4\textsuperscript{(*)} & \textbf{5977.0} & 9.5 & 30/30 & 1\textsuperscript{(+)}3\textsuperscript{(*)}4\textsuperscript{(+)} & 5975.0 & 13.1 & 30/30 & 1\textsuperscript{(+)}2\textsuperscript{(*)}4\textsuperscript{(+)} & 5949.3 & 26.8 & 30/30 & 1\textsuperscript{(*)}2\textsuperscript{(-)}3\textsuperscript{(-)} \\
 &  & 0.80 & - & - & 0/30 &  & - & - & 0/30 &  & - & - & 0/30 &  & - & - & 0/30 &  \\
 &  & 0.90 & - & - & 0/30 &  & - & - & 0/30 &  & - & - & 0/30 &  & - & - & 0/30 &  \\
\midrule
opm5 & 0.10 & 0.60 & 5877.0 & 78.9 & 27/30 & 2\textsuperscript{(-)}3\textsuperscript{(-)}4\textsuperscript{(*)} & \textbf{5979.3} & 11.1 & 30/30 & 1\textsuperscript{(+)}3\textsuperscript{(*)}4\textsuperscript{(+)} & \textbf{5979.3} & 12.3 & 30/30 & 1\textsuperscript{(+)}2\textsuperscript{(*)}4\textsuperscript{(+)} & 5947.7 & 34.8 & 30/30 & 1\textsuperscript{(*)}2\textsuperscript{(-)}3\textsuperscript{(-)} \\
 &  & 0.70 & 5877.0 & 78.9 & 27/30 & 2\textsuperscript{(-)}3\textsuperscript{(-)}4\textsuperscript{(*)} & \textbf{5979.3} & 11.1 & 30/30 & 1\textsuperscript{(+)}3\textsuperscript{(*)}4\textsuperscript{(+)} & \textbf{5979.3} & 12.3 & 30/30 & 1\textsuperscript{(+)}2\textsuperscript{(*)}4\textsuperscript{(+)} & 5947.7 & 34.8 & 30/30 & 1\textsuperscript{(*)}2\textsuperscript{(-)}3\textsuperscript{(-)} \\
 &  & 0.80 & 5877.0 & 78.9 & 27/30 & 2\textsuperscript{(-)}3\textsuperscript{(-)}4\textsuperscript{(*)} & \textbf{5979.3} & 11.1 & 30/30 & 1\textsuperscript{(+)}3\textsuperscript{(*)}4\textsuperscript{(+)} & \textbf{5979.3} & 12.3 & 30/30 & 1\textsuperscript{(+)}2\textsuperscript{(*)}4\textsuperscript{(+)} & 5947.7 & 34.8 & 30/30 & 1\textsuperscript{(*)}2\textsuperscript{(-)}3\textsuperscript{(-)} \\
 &  & 0.90 & 5877.0 & 78.9 & 27/30 & 2\textsuperscript{(-)}3\textsuperscript{(-)}4\textsuperscript{(*)} & \textbf{5979.3} & 11.1 & 30/30 & 1\textsuperscript{(+)}3\textsuperscript{(*)}4\textsuperscript{(+)} & \textbf{5979.3} & 12.3 & 30/30 & 1\textsuperscript{(+)}2\textsuperscript{(*)}4\textsuperscript{(+)} & 5947.7 & 34.8 & 30/30 & 1\textsuperscript{(*)}2\textsuperscript{(-)}3\textsuperscript{(-)} \\
\cmidrule(lr){2-19}
 & 0.12 & 0.60 & 5877.0 & 78.9 & 27/30 & 2\textsuperscript{(-)}3\textsuperscript{(-)}4\textsuperscript{(*)} & \textbf{5979.3} & 11.1 & 30/30 & 1\textsuperscript{(+)}3\textsuperscript{(*)}4\textsuperscript{(+)} & \textbf{5979.3} & 12.3 & 30/30 & 1\textsuperscript{(+)}2\textsuperscript{(*)}4\textsuperscript{(+)} & 5947.7 & 34.8 & 30/30 & 1\textsuperscript{(*)}2\textsuperscript{(-)}3\textsuperscript{(-)} \\
 &  & 0.70 & 5877.0 & 78.9 & 27/30 & 2\textsuperscript{(-)}3\textsuperscript{(-)}4\textsuperscript{(*)} & \textbf{5979.3} & 11.1 & 30/30 & 1\textsuperscript{(+)}3\textsuperscript{(*)}4\textsuperscript{(+)} & \textbf{5979.3} & 12.3 & 30/30 & 1\textsuperscript{(+)}2\textsuperscript{(*)}4\textsuperscript{(+)} & 5947.7 & 34.8 & 30/30 & 1\textsuperscript{(*)}2\textsuperscript{(-)}3\textsuperscript{(-)} \\
 &  & 0.80 & 5877.0 & 78.9 & 27/30 & 2\textsuperscript{(-)}3\textsuperscript{(-)}4\textsuperscript{(*)} & \textbf{5979.3} & 11.1 & 30/30 & 1\textsuperscript{(+)}3\textsuperscript{(*)}4\textsuperscript{(+)} & \textbf{5979.3} & 12.3 & 30/30 & 1\textsuperscript{(+)}2\textsuperscript{(*)}4\textsuperscript{(+)} & 5947.7 & 34.8 & 30/30 & 1\textsuperscript{(*)}2\textsuperscript{(-)}3\textsuperscript{(-)} \\
 &  & 0.90 & 5877.0 & 78.9 & 27/30 & 2\textsuperscript{(-)}3\textsuperscript{(-)}4\textsuperscript{(*)} & \textbf{5979.3} & 11.1 & 30/30 & 1\textsuperscript{(+)}3\textsuperscript{(*)}4\textsuperscript{(+)} & \textbf{5979.3} & 12.3 & 30/30 & 1\textsuperscript{(+)}2\textsuperscript{(*)}4\textsuperscript{(+)} & 5947.7 & 34.8 & 30/30 & 1\textsuperscript{(*)}2\textsuperscript{(-)}3\textsuperscript{(-)} \\
\cmidrule(lr){2-19}
 & 0.15 & 0.60 & 5877.0 & 78.9 & 27/30 & 2\textsuperscript{(-)}3\textsuperscript{(-)}4\textsuperscript{(*)} & \textbf{5979.3} & 11.1 & 30/30 & 1\textsuperscript{(+)}3\textsuperscript{(*)}4\textsuperscript{(+)} & \textbf{5979.3} & 12.3 & 30/30 & 1\textsuperscript{(+)}2\textsuperscript{(*)}4\textsuperscript{(+)} & 5947.7 & 34.8 & 30/30 & 1\textsuperscript{(*)}2\textsuperscript{(-)}3\textsuperscript{(-)} \\
 &  & 0.70 & 5877.0 & 78.9 & 27/30 & 2\textsuperscript{(-)}3\textsuperscript{(-)}4\textsuperscript{(*)} & \textbf{5979.3} & 11.1 & 30/30 & 1\textsuperscript{(+)}3\textsuperscript{(*)}4\textsuperscript{(+)} & \textbf{5979.3} & 12.3 & 30/30 & 1\textsuperscript{(+)}2\textsuperscript{(*)}4\textsuperscript{(+)} & 5947.7 & 34.8 & 30/30 & 1\textsuperscript{(*)}2\textsuperscript{(-)}3\textsuperscript{(-)} \\
 &  & 0.80 & 5877.0 & 78.9 & 27/30 & 2\textsuperscript{(-)}3\textsuperscript{(-)}4\textsuperscript{(*)} & \textbf{5979.3} & 11.1 & 30/30 & 1\textsuperscript{(+)}3\textsuperscript{(*)}4\textsuperscript{(+)} & \textbf{5979.3} & 12.3 & 30/30 & 1\textsuperscript{(+)}2\textsuperscript{(*)}4\textsuperscript{(+)} & 5947.7 & 34.8 & 30/30 & 1\textsuperscript{(*)}2\textsuperscript{(-)}3\textsuperscript{(-)} \\
 &  & 0.90 & 5877.0 & 78.9 & 27/30 & 2\textsuperscript{(-)}3\textsuperscript{(-)}4\textsuperscript{(*)} & 5978.0 & 13.2 & 30/30 & 1\textsuperscript{(+)}3\textsuperscript{(*)}4\textsuperscript{(+)} & \textbf{5980.7} & 9.8 & 30/30 & 1\textsuperscript{(+)}2\textsuperscript{(*)}4\textsuperscript{(+)} & 5947.7 & 34.8 & 30/30 & 1\textsuperscript{(*)}2\textsuperscript{(-)}3\textsuperscript{(-)} \\
\midrule
opm6 & 0.10 & 0.60 & 5872.3 & 81.6 & 26/30 & 2\textsuperscript{(-)}3\textsuperscript{(-)}4\textsuperscript{(*)} & 5973.7 & 12.5 & 30/30 & 1\textsuperscript{(+)}3\textsuperscript{(*)}4\textsuperscript{(+)} & \textbf{5976.3} & 8.5 & 30/30 & 1\textsuperscript{(+)}2\textsuperscript{(*)}4\textsuperscript{(+)} & 5946.7 & 30.9 & 30/30 & 1\textsuperscript{(*)}2\textsuperscript{(-)}3\textsuperscript{(-)} \\
 &  & 0.70 & 5872.3 & 81.6 & 26/30 & 2\textsuperscript{(-)}3\textsuperscript{(-)}4\textsuperscript{(*)} & 5973.7 & 12.5 & 30/30 & 1\textsuperscript{(+)}3\textsuperscript{(*)}4\textsuperscript{(+)} & \textbf{5976.3} & 8.5 & 30/30 & 1\textsuperscript{(+)}2\textsuperscript{(*)}4\textsuperscript{(+)} & 5946.7 & 30.9 & 30/30 & 1\textsuperscript{(*)}2\textsuperscript{(-)}3\textsuperscript{(-)} \\
 &  & 0.80 & 5872.3 & 81.6 & 26/30 & 2\textsuperscript{(-)}3\textsuperscript{(-)}4\textsuperscript{(*)} & 5973.0 & 12.6 & 30/30 & 1\textsuperscript{(+)}3\textsuperscript{(*)}4\textsuperscript{(+)} & \textbf{5975.0} & 9.4 & 30/30 & 1\textsuperscript{(+)}2\textsuperscript{(*)}4\textsuperscript{(+)} & 5946.7 & 30.9 & 30/30 & 1\textsuperscript{(*)}2\textsuperscript{(-)}3\textsuperscript{(-)} \\
 &  & 0.90 & 5870.8 & 84.0 & 24/30 & 2\textsuperscript{(-)}3\textsuperscript{(-)}4\textsuperscript{(*)} & \textbf{5974.7} & 10.4 & 30/30 & 1\textsuperscript{(+)}3\textsuperscript{(*)}4\textsuperscript{(+)} & 5972.0 & 14.5 & 30/30 & 1\textsuperscript{(+)}2\textsuperscript{(*)}4\textsuperscript{(+)} & 5943.3 & 31.7 & 30/30 & 1\textsuperscript{(*)}2\textsuperscript{(-)}3\textsuperscript{(-)} \\
\cmidrule(lr){2-19}
 & 0.12 & 0.60 & 5872.3 & 81.6 & 26/30 & 2\textsuperscript{(-)}3\textsuperscript{(-)}4\textsuperscript{(*)} & 5974.0 & 13.0 & 30/30 & 1\textsuperscript{(+)}3\textsuperscript{(*)}4\textsuperscript{(+)} & \textbf{5976.7} & 9.2 & 30/30 & 1\textsuperscript{(+)}2\textsuperscript{(*)}4\textsuperscript{(+)} & 5949.0 & 29.8 & 30/30 & 1\textsuperscript{(*)}2\textsuperscript{(-)}3\textsuperscript{(-)} \\
 &  & 0.70 & 5872.3 & 81.6 & 26/30 & 2\textsuperscript{(-)}3\textsuperscript{(-)}4\textsuperscript{(*)} & 5974.0 & 13.0 & 30/30 & 1\textsuperscript{(+)}3\textsuperscript{(*)}4\textsuperscript{(+)} & \textbf{5976.7} & 9.2 & 30/30 & 1\textsuperscript{(+)}2\textsuperscript{(*)}4\textsuperscript{(+)} & 5949.0 & 29.8 & 30/30 & 1\textsuperscript{(*)}2\textsuperscript{(-)}3\textsuperscript{(-)} \\
 &  & 0.80 & 5871.9 & 80.9 & 26/30 & 2\textsuperscript{(-)}3\textsuperscript{(-)}4\textsuperscript{(*)} & \textbf{5974.0} & 11.9 & 30/30 & 1\textsuperscript{(+)}3\textsuperscript{(*)}4\textsuperscript{(+)} & 5973.3 & 14.2 & 30/30 & 1\textsuperscript{(+)}2\textsuperscript{(*)}4\textsuperscript{(+)} & 5949.0 & 29.8 & 30/30 & 1\textsuperscript{(*)}2\textsuperscript{(-)}3\textsuperscript{(-)} \\
 &  & 0.90 & 5870.0 & 82.3 & 10/30 & 2\textsuperscript{(-)}3\textsuperscript{(*)}4\textsuperscript{(*)} & \textbf{5947.0} & 37.2 & 30/30 & 1\textsuperscript{(+)}3\textsuperscript{(*)}4\textsuperscript{(+)} & 5936.3 & 36.3 & 30/30 & 1\textsuperscript{(*)}2\textsuperscript{(*)}4\textsuperscript{(+)} & 5899.0 & 45.5 & 30/30 & 1\textsuperscript{(*)}2\textsuperscript{(-)}3\textsuperscript{(-)} \\
\cmidrule(lr){2-19}
 & 0.15 & 0.60 & 5872.3 & 81.6 & 26/30 & 2\textsuperscript{(-)}3\textsuperscript{(-)}4\textsuperscript{(*)} & 5973.3 & 13.2 & 30/30 & 1\textsuperscript{(+)}3\textsuperscript{(*)}4\textsuperscript{(+)} & \textbf{5975.3} & 8.6 & 30/30 & 1\textsuperscript{(+)}2\textsuperscript{(*)}4\textsuperscript{(+)} & 5945.7 & 34.2 & 30/30 & 1\textsuperscript{(*)}2\textsuperscript{(-)}3\textsuperscript{(-)} \\
 &  & 0.70 & 5871.9 & 81.0 & 26/30 & 2\textsuperscript{(-)}3\textsuperscript{(-)}4\textsuperscript{(*)} & \textbf{5974.0} & 11.9 & 30/30 & 1\textsuperscript{(+)}3\textsuperscript{(*)}4\textsuperscript{(+)} & \textbf{5974.0} & 14.0 & 30/30 & 1\textsuperscript{(+)}2\textsuperscript{(*)}4\textsuperscript{(+)} & 5947.7 & 30.6 & 30/30 & 1\textsuperscript{(*)}2\textsuperscript{(-)}3\textsuperscript{(-)} \\
 &  & 0.80 & 5870.0 & 82.3 & 10/30 & 2\textsuperscript{(-)}3\textsuperscript{(-)}4\textsuperscript{(*)} & \textbf{5955.0} & 22.4 & 30/30 & 1\textsuperscript{(+)}3\textsuperscript{(*)}4\textsuperscript{(+)} & 5954.7 & 26.2 & 30/30 & 1\textsuperscript{(+)}2\textsuperscript{(*)}4\textsuperscript{(+)} & 5906.7 & 43.4 & 30/30 & 1\textsuperscript{(*)}2\textsuperscript{(-)}3\textsuperscript{(-)} \\
 &  & 0.90 & 5842.9 & 53.5 & 7/30 & 2\textsuperscript{(-)}3\textsuperscript{(-)}4\textsuperscript{(*)} & \textbf{5957.3} & 24.5 & 30/30 & 1\textsuperscript{(+)}3\textsuperscript{(*)}4\textsuperscript{(+)} & 5953.3 & 27.5 & 30/30 & 1\textsuperscript{(+)}2\textsuperscript{(*)}4\textsuperscript{(+)} & 5915.0 & 43.5 & 30/30 & 1\textsuperscript{(*)}2\textsuperscript{(-)}3\textsuperscript{(-)} \\
\bottomrule
\end{tabular}
}
\end{table}

\begin{table}[!t]
\centering
\caption{Mean, standard deviation, statistical comparisons, and feasible-run count for the best feasible ore production under S3 across all $\kappa$ values.}
\label{tab:f1_setting_S3}
\resizebox{\textwidth}{!}{
\begin{tabular}{l c c rrrr rrrr rrrr rrrr}
\toprule
Instance & $\kappa$ & $\alpha$ & \multicolumn{4}{c}{GSEMO\textsuperscript{(1)}} & \multicolumn{4}{c}{NSGA-II\textsuperscript{(2)}} & \multicolumn{4}{c}{NSGA-III\textsuperscript{(3)}} & \multicolumn{4}{c}{MOEA/D\textsuperscript{(4)}} \\
\cmidrule(lr){4-7} \cmidrule(lr){8-11} \cmidrule(lr){12-15} \cmidrule(lr){16-19}
 & & & Mean & Std & Feas. & Stat & Mean & Std & Feas. & Stat & Mean & Std & Feas. & Stat & Mean & Std & Feas. & Stat \\
\midrule
opm1 & 0.10 & 0.60 & 5879.2 & 77.9 & 26/30 & 2\textsuperscript{(-)}3\textsuperscript{(-)}4\textsuperscript{(*)} & \textbf{5980.0} & 0.0 & 30/30 & 1\textsuperscript{(+)}3\textsuperscript{(*)}4\textsuperscript{(+)} & 5975.0 & 13.1 & 30/30 & 1\textsuperscript{(+)}2\textsuperscript{(*)}4\textsuperscript{(+)} & 5939.0 & 34.7 & 30/30 & 1\textsuperscript{(*)}2\textsuperscript{(-)}3\textsuperscript{(-)} \\
 &  & 0.70 & 5879.2 & 77.9 & 26/30 & 2\textsuperscript{(-)}3\textsuperscript{(-)}4\textsuperscript{(*)} & \textbf{5980.0} & 0.0 & 30/30 & 1\textsuperscript{(+)}3\textsuperscript{(*)}4\textsuperscript{(+)} & 5975.0 & 13.1 & 30/30 & 1\textsuperscript{(+)}2\textsuperscript{(*)}4\textsuperscript{(+)} & 5939.0 & 34.7 & 30/30 & 1\textsuperscript{(*)}2\textsuperscript{(-)}3\textsuperscript{(-)} \\
 &  & 0.80 & 5879.2 & 77.9 & 26/30 & 2\textsuperscript{(-)}3\textsuperscript{(-)}4\textsuperscript{(*)} & \textbf{5980.0} & 0.0 & 30/30 & 1\textsuperscript{(+)}3\textsuperscript{(*)}4\textsuperscript{(+)} & 5975.0 & 13.1 & 30/30 & 1\textsuperscript{(+)}2\textsuperscript{(*)}4\textsuperscript{(+)} & 5939.0 & 34.7 & 30/30 & 1\textsuperscript{(*)}2\textsuperscript{(-)}3\textsuperscript{(-)} \\
 &  & 0.90 & 5882.0 & 77.0 & 20/30 & 2\textsuperscript{(-)}3\textsuperscript{(-)}4\textsuperscript{(*)} & \textbf{5977.3} & 8.7 & 30/30 & 1\textsuperscript{(+)}3\textsuperscript{(*)}4\textsuperscript{(+)} & 5975.7 & 12.8 & 30/30 & 1\textsuperscript{(+)}2\textsuperscript{(*)}4\textsuperscript{(+)} & 5939.0 & 34.7 & 30/30 & 1\textsuperscript{(*)}2\textsuperscript{(-)}3\textsuperscript{(-)} \\
\cmidrule(lr){2-19}
 & 0.12 & 0.60 & 5879.2 & 77.9 & 26/30 & 2\textsuperscript{(-)}3\textsuperscript{(-)}4\textsuperscript{(*)} & \textbf{5980.0} & 0.0 & 30/30 & 1\textsuperscript{(+)}3\textsuperscript{(*)}4\textsuperscript{(+)} & 5975.0 & 13.1 & 30/30 & 1\textsuperscript{(+)}2\textsuperscript{(*)}4\textsuperscript{(+)} & 5939.0 & 34.7 & 30/30 & 1\textsuperscript{(*)}2\textsuperscript{(-)}3\textsuperscript{(-)} \\
 &  & 0.70 & 5879.2 & 77.9 & 26/30 & 2\textsuperscript{(-)}3\textsuperscript{(-)}4\textsuperscript{(*)} & \textbf{5980.0} & 0.0 & 30/30 & 1\textsuperscript{(+)}3\textsuperscript{(*)}4\textsuperscript{(+)} & 5975.0 & 13.1 & 30/30 & 1\textsuperscript{(+)}2\textsuperscript{(*)}4\textsuperscript{(+)} & 5939.0 & 34.7 & 30/30 & 1\textsuperscript{(*)}2\textsuperscript{(-)}3\textsuperscript{(-)} \\
 &  & 0.80 & 5879.0 & 78.6 & 21/30 & 2\textsuperscript{(-)}3\textsuperscript{(-)}4\textsuperscript{(*)} & \textbf{5978.3} & 6.5 & 30/30 & 1\textsuperscript{(+)}3\textsuperscript{(*)}4\textsuperscript{(+)} & 5975.0 & 13.1 & 30/30 & 1\textsuperscript{(+)}2\textsuperscript{(*)}4\textsuperscript{(+)} & 5939.0 & 34.7 & 30/30 & 1\textsuperscript{(*)}2\textsuperscript{(-)}3\textsuperscript{(-)} \\
 &  & 0.90 & - & - & 0/30 &  & - & - & 0/30 &  & - & - & 0/30 &  & - & - & 0/30 &  \\
\cmidrule(lr){2-19}
 & 0.15 & 0.60 & 5879.2 & 77.9 & 26/30 & 2\textsuperscript{(-)}3\textsuperscript{(-)}4\textsuperscript{(*)} & \textbf{5980.0} & 0.0 & 30/30 & 1\textsuperscript{(+)}3\textsuperscript{(*)}4\textsuperscript{(+)} & 5975.0 & 13.1 & 30/30 & 1\textsuperscript{(+)}2\textsuperscript{(*)}4\textsuperscript{(+)} & 5939.0 & 34.7 & 30/30 & 1\textsuperscript{(*)}2\textsuperscript{(-)}3\textsuperscript{(-)} \\
 &  & 0.70 & 5879.0 & 78.6 & 21/30 & 2\textsuperscript{(-)}3\textsuperscript{(-)}4\textsuperscript{(*)} & \textbf{5978.3} & 6.5 & 30/30 & 1\textsuperscript{(+)}3\textsuperscript{(*)}4\textsuperscript{(+)} & 5975.7 & 12.8 & 30/30 & 1\textsuperscript{(+)}2\textsuperscript{(*)}4\textsuperscript{(+)} & 5939.0 & 34.7 & 30/30 & 1\textsuperscript{(*)}2\textsuperscript{(-)}3\textsuperscript{(-)} \\
 &  & 0.80 & - & - & 0/30 &  & - & - & 0/30 &  & - & - & 0/30 &  & - & - & 0/30 &  \\
 &  & 0.90 & - & - & 0/30 &  & - & - & 0/30 &  & - & - & 0/30 &  & - & - & 0/30 &  \\
\midrule
opm2 & 0.10 & 0.60 & 5883.6 & 79.4 & 25/30 & 2\textsuperscript{(-)}3\textsuperscript{(-)}4\textsuperscript{(*)} & 5976.7 & 10.6 & 30/30 & 1\textsuperscript{(+)}3\textsuperscript{(*)}4\textsuperscript{(+)} & \textbf{5978.0} & 6.1 & 30/30 & 1\textsuperscript{(+)}2\textsuperscript{(*)}4\textsuperscript{(+)} & 5942.3 & 30.8 & 30/30 & 1\textsuperscript{(*)}2\textsuperscript{(-)}3\textsuperscript{(-)} \\
 &  & 0.70 & 5883.6 & 79.4 & 25/30 & 2\textsuperscript{(-)}3\textsuperscript{(-)}4\textsuperscript{(*)} & 5976.7 & 10.6 & 30/30 & 1\textsuperscript{(+)}3\textsuperscript{(*)}4\textsuperscript{(+)} & \textbf{5978.0} & 6.1 & 30/30 & 1\textsuperscript{(+)}2\textsuperscript{(*)}4\textsuperscript{(+)} & 5942.3 & 30.8 & 30/30 & 1\textsuperscript{(*)}2\textsuperscript{(-)}3\textsuperscript{(-)} \\
 &  & 0.80 & 5883.6 & 79.4 & 25/30 & 2\textsuperscript{(-)}3\textsuperscript{(-)}4\textsuperscript{(*)} & 5976.7 & 10.6 & 30/30 & 1\textsuperscript{(+)}3\textsuperscript{(*)}4\textsuperscript{(+)} & \textbf{5978.0} & 6.1 & 30/30 & 1\textsuperscript{(+)}2\textsuperscript{(*)}4\textsuperscript{(+)} & 5944.3 & 28.0 & 30/30 & 1\textsuperscript{(*)}2\textsuperscript{(-)}3\textsuperscript{(-)} \\
 &  & 0.90 & \textbf{5900.0} & 90.3 & 5/30 &  & - & - & 0/30 &  & - & - & 0/30 &  & - & - & 0/30 &  \\
\cmidrule(lr){2-19}
 & 0.12 & 0.60 & 5887.6 & 77.5 & 25/30 & 2\textsuperscript{(-)}3\textsuperscript{(-)}4\textsuperscript{(*)} & 5978.0 & 8.1 & 30/30 & 1\textsuperscript{(+)}3\textsuperscript{(*)}4\textsuperscript{(+)} & \textbf{5978.7} & 5.1 & 30/30 & 1\textsuperscript{(+)}2\textsuperscript{(*)}4\textsuperscript{(+)} & 5944.3 & 29.4 & 30/30 & 1\textsuperscript{(*)}2\textsuperscript{(-)}3\textsuperscript{(-)} \\
 &  & 0.70 & 5887.6 & 77.5 & 25/30 & 2\textsuperscript{(-)}3\textsuperscript{(-)}4\textsuperscript{(*)} & 5978.0 & 8.1 & 30/30 & 1\textsuperscript{(+)}3\textsuperscript{(*)}4\textsuperscript{(+)} & \textbf{5978.7} & 5.1 & 30/30 & 1\textsuperscript{(+)}2\textsuperscript{(*)}4\textsuperscript{(+)} & 5944.3 & 29.4 & 30/30 & 1\textsuperscript{(*)}2\textsuperscript{(-)}3\textsuperscript{(-)} \\
 &  & 0.80 & \textbf{5873.3} & 79.2 & 9/30 & 2\textsuperscript{(*)}3\textsuperscript{(*)}4\textsuperscript{(*)} & 5820.0 & 34.6 & 3/30 & 1\textsuperscript{(*)}3\textsuperscript{(*)}4\textsuperscript{(*)} & 5830.0 & 42.4 & 2/30 & 1\textsuperscript{(*)}2\textsuperscript{(*)}4\textsuperscript{(*)} & 5800.0 & 0.0 & 1/30 & 1\textsuperscript{(*)}2\textsuperscript{(*)}3\textsuperscript{(*)} \\
 &  & 0.90 & - & - & 0/30 &  & - & - & 0/30 &  & - & - & 0/30 &  & - & - & 0/30 &  \\
\cmidrule(lr){2-19}
 & 0.15 & 0.60 & 5884.2 & 77.9 & 26/30 & 2\textsuperscript{(-)}3\textsuperscript{(-)}4\textsuperscript{(*)} & 5976.0 & 11.0 & 30/30 & 1\textsuperscript{(+)}3\textsuperscript{(*)}4\textsuperscript{(+)} & \textbf{5976.7} & 11.8 & 30/30 & 1\textsuperscript{(+)}2\textsuperscript{(*)}4\textsuperscript{(+)} & 5941.7 & 27.3 & 30/30 & 1\textsuperscript{(*)}2\textsuperscript{(-)}3\textsuperscript{(-)} \\
 &  & 0.70 & \textbf{5871.2} & 82.7 & 8/30 &  & - & - & 0/30 &  & - & - & 0/30 &  & - & - & 0/30 &  \\
 &  & 0.80 & - & - & 0/30 &  & - & - & 0/30 &  & - & - & 0/30 &  & - & - & 0/30 &  \\
 &  & 0.90 & - & - & 0/30 &  & - & - & 0/30 &  & - & - & 0/30 &  & - & - & 0/30 &  \\
\midrule
opm5 & 0.10 & 0.60 & 5879.2 & 77.9 & 26/30 & 2\textsuperscript{(-)}3\textsuperscript{(-)}4\textsuperscript{(*)} & \textbf{5980.0} & 0.0 & 30/30 & 1\textsuperscript{(+)}3\textsuperscript{(*)}4\textsuperscript{(+)} & 5975.0 & 13.1 & 30/30 & 1\textsuperscript{(+)}2\textsuperscript{(*)}4\textsuperscript{(+)} & 5939.0 & 34.7 & 30/30 & 1\textsuperscript{(*)}2\textsuperscript{(-)}3\textsuperscript{(-)} \\
 &  & 0.70 & 5879.2 & 77.9 & 26/30 & 2\textsuperscript{(-)}3\textsuperscript{(-)}4\textsuperscript{(*)} & \textbf{5980.0} & 0.0 & 30/30 & 1\textsuperscript{(+)}3\textsuperscript{(*)}4\textsuperscript{(+)} & 5975.0 & 13.1 & 30/30 & 1\textsuperscript{(+)}2\textsuperscript{(*)}4\textsuperscript{(+)} & 5939.0 & 34.7 & 30/30 & 1\textsuperscript{(*)}2\textsuperscript{(-)}3\textsuperscript{(-)} \\
 &  & 0.80 & 5879.2 & 77.9 & 26/30 & 2\textsuperscript{(-)}3\textsuperscript{(-)}4\textsuperscript{(*)} & \textbf{5980.0} & 0.0 & 30/30 & 1\textsuperscript{(+)}3\textsuperscript{(*)}4\textsuperscript{(+)} & 5975.0 & 13.1 & 30/30 & 1\textsuperscript{(+)}2\textsuperscript{(*)}4\textsuperscript{(+)} & 5939.0 & 34.7 & 30/30 & 1\textsuperscript{(*)}2\textsuperscript{(-)}3\textsuperscript{(-)} \\
 &  & 0.90 & 5879.2 & 77.9 & 26/30 & 2\textsuperscript{(-)}3\textsuperscript{(-)}4\textsuperscript{(*)} & \textbf{5980.0} & 0.0 & 30/30 & 1\textsuperscript{(+)}3\textsuperscript{(*)}4\textsuperscript{(+)} & 5975.0 & 13.1 & 30/30 & 1\textsuperscript{(+)}2\textsuperscript{(*)}4\textsuperscript{(+)} & 5939.0 & 34.7 & 30/30 & 1\textsuperscript{(*)}2\textsuperscript{(-)}3\textsuperscript{(-)} \\
\cmidrule(lr){2-19}
 & 0.12 & 0.60 & 5879.2 & 77.9 & 26/30 & 2\textsuperscript{(-)}3\textsuperscript{(-)}4\textsuperscript{(*)} & \textbf{5980.0} & 0.0 & 30/30 & 1\textsuperscript{(+)}3\textsuperscript{(*)}4\textsuperscript{(+)} & 5975.0 & 13.1 & 30/30 & 1\textsuperscript{(+)}2\textsuperscript{(*)}4\textsuperscript{(+)} & 5939.0 & 34.7 & 30/30 & 1\textsuperscript{(*)}2\textsuperscript{(-)}3\textsuperscript{(-)} \\
 &  & 0.70 & 5879.2 & 77.9 & 26/30 & 2\textsuperscript{(-)}3\textsuperscript{(-)}4\textsuperscript{(*)} & \textbf{5980.0} & 0.0 & 30/30 & 1\textsuperscript{(+)}3\textsuperscript{(*)}4\textsuperscript{(+)} & 5975.0 & 13.1 & 30/30 & 1\textsuperscript{(+)}2\textsuperscript{(*)}4\textsuperscript{(+)} & 5939.0 & 34.7 & 30/30 & 1\textsuperscript{(*)}2\textsuperscript{(-)}3\textsuperscript{(-)} \\
 &  & 0.80 & 5879.2 & 77.9 & 26/30 & 2\textsuperscript{(-)}3\textsuperscript{(-)}4\textsuperscript{(*)} & \textbf{5980.0} & 0.0 & 30/30 & 1\textsuperscript{(+)}3\textsuperscript{(*)}4\textsuperscript{(+)} & 5975.0 & 13.1 & 30/30 & 1\textsuperscript{(+)}2\textsuperscript{(*)}4\textsuperscript{(+)} & 5939.0 & 34.7 & 30/30 & 1\textsuperscript{(*)}2\textsuperscript{(-)}3\textsuperscript{(-)} \\
 &  & 0.90 & 5879.2 & 77.9 & 26/30 & 2\textsuperscript{(-)}3\textsuperscript{(-)}4\textsuperscript{(*)} & \textbf{5980.0} & 0.0 & 30/30 & 1\textsuperscript{(+)}3\textsuperscript{(*)}4\textsuperscript{(+)} & 5975.0 & 13.1 & 30/30 & 1\textsuperscript{(+)}2\textsuperscript{(*)}4\textsuperscript{(+)} & 5939.0 & 34.7 & 30/30 & 1\textsuperscript{(*)}2\textsuperscript{(-)}3\textsuperscript{(-)} \\
\cmidrule(lr){2-19}
 & 0.15 & 0.60 & 5879.2 & 77.9 & 26/30 & 2\textsuperscript{(-)}3\textsuperscript{(-)}4\textsuperscript{(*)} & \textbf{5980.0} & 0.0 & 30/30 & 1\textsuperscript{(+)}3\textsuperscript{(*)}4\textsuperscript{(+)} & 5975.0 & 13.1 & 30/30 & 1\textsuperscript{(+)}2\textsuperscript{(*)}4\textsuperscript{(+)} & 5939.0 & 34.7 & 30/30 & 1\textsuperscript{(*)}2\textsuperscript{(-)}3\textsuperscript{(-)} \\
 &  & 0.70 & 5879.2 & 77.9 & 26/30 & 2\textsuperscript{(-)}3\textsuperscript{(-)}4\textsuperscript{(*)} & \textbf{5980.0} & 0.0 & 30/30 & 1\textsuperscript{(+)}3\textsuperscript{(*)}4\textsuperscript{(+)} & 5975.0 & 13.1 & 30/30 & 1\textsuperscript{(+)}2\textsuperscript{(*)}4\textsuperscript{(+)} & 5939.0 & 34.7 & 30/30 & 1\textsuperscript{(*)}2\textsuperscript{(-)}3\textsuperscript{(-)} \\
 &  & 0.80 & 5879.2 & 77.9 & 26/30 & 2\textsuperscript{(-)}3\textsuperscript{(-)}4\textsuperscript{(*)} & \textbf{5980.0} & 0.0 & 30/30 & 1\textsuperscript{(+)}3\textsuperscript{(*)}4\textsuperscript{(+)} & 5975.0 & 13.1 & 30/30 & 1\textsuperscript{(+)}2\textsuperscript{(*)}4\textsuperscript{(+)} & 5939.0 & 34.7 & 30/30 & 1\textsuperscript{(*)}2\textsuperscript{(-)}3\textsuperscript{(-)} \\
 &  & 0.90 & 5879.0 & 78.6 & 21/30 & 2\textsuperscript{(-)}3\textsuperscript{(-)}4\textsuperscript{(*)} & \textbf{5979.0} & 5.5 & 30/30 & 1\textsuperscript{(+)}3\textsuperscript{(*)}4\textsuperscript{(+)} & 5975.7 & 12.8 & 30/30 & 1\textsuperscript{(+)}2\textsuperscript{(*)}4\textsuperscript{(+)} & 5939.0 & 34.7 & 30/30 & 1\textsuperscript{(*)}2\textsuperscript{(-)}3\textsuperscript{(-)} \\
\midrule
opm6 & 0.10 & 0.60 & 5884.6 & 83.4 & 26/30 & 2\textsuperscript{(-)}3\textsuperscript{(-)}4\textsuperscript{(*)} & 5975.0 & 12.0 & 30/30 & 1\textsuperscript{(+)}3\textsuperscript{(*)}4\textsuperscript{(+)} & \textbf{5976.7} & 10.6 & 30/30 & 1\textsuperscript{(+)}2\textsuperscript{(*)}4\textsuperscript{(+)} & 5943.7 & 31.8 & 30/30 & 1\textsuperscript{(*)}2\textsuperscript{(-)}3\textsuperscript{(-)} \\
 &  & 0.70 & 5884.6 & 83.4 & 26/30 & 2\textsuperscript{(-)}3\textsuperscript{(-)}4\textsuperscript{(*)} & 5975.0 & 12.0 & 30/30 & 1\textsuperscript{(+)}3\textsuperscript{(*)}4\textsuperscript{(+)} & \textbf{5976.7} & 10.6 & 30/30 & 1\textsuperscript{(+)}2\textsuperscript{(*)}4\textsuperscript{(+)} & 5943.7 & 31.8 & 30/30 & 1\textsuperscript{(*)}2\textsuperscript{(-)}3\textsuperscript{(-)} \\
 &  & 0.80 & 5884.6 & 83.4 & 26/30 & 2\textsuperscript{(-)}3\textsuperscript{(-)}4\textsuperscript{(*)} & 5975.0 & 12.0 & 30/30 & 1\textsuperscript{(+)}3\textsuperscript{(*)}4\textsuperscript{(+)} & \textbf{5976.7} & 10.6 & 30/30 & 1\textsuperscript{(+)}2\textsuperscript{(*)}4\textsuperscript{(+)} & 5943.7 & 31.8 & 30/30 & 1\textsuperscript{(*)}2\textsuperscript{(-)}3\textsuperscript{(-)} \\
 &  & 0.90 & 5884.6 & 83.4 & 26/30 & 2\textsuperscript{(-)}3\textsuperscript{(-)}4\textsuperscript{(*)} & 5974.3 & 12.2 & 30/30 & 1\textsuperscript{(+)}3\textsuperscript{(*)}4\textsuperscript{(+)} & \textbf{5975.3} & 11.4 & 30/30 & 1\textsuperscript{(+)}2\textsuperscript{(*)}4\textsuperscript{(+)} & 5943.7 & 31.8 & 30/30 & 1\textsuperscript{(*)}2\textsuperscript{(-)}3\textsuperscript{(-)} \\
\cmidrule(lr){2-19}
 & 0.12 & 0.60 & 5876.2 & 80.4 & 26/30 & 2\textsuperscript{(-)}3\textsuperscript{(-)}4\textsuperscript{(*)} & \textbf{5978.7} & 9.0 & 30/30 & 1\textsuperscript{(+)}3\textsuperscript{(*)}4\textsuperscript{(+)} & 5974.0 & 13.3 & 30/30 & 1\textsuperscript{(+)}2\textsuperscript{(*)}4\textsuperscript{(+)} & 5945.0 & 30.3 & 30/30 & 1\textsuperscript{(*)}2\textsuperscript{(-)}3\textsuperscript{(-)} \\
 &  & 0.70 & 5876.2 & 80.4 & 26/30 & 2\textsuperscript{(-)}3\textsuperscript{(-)}4\textsuperscript{(*)} & \textbf{5978.7} & 9.0 & 30/30 & 1\textsuperscript{(+)}3\textsuperscript{(*)}4\textsuperscript{(+)} & 5974.0 & 13.3 & 30/30 & 1\textsuperscript{(+)}2\textsuperscript{(*)}4\textsuperscript{(+)} & 5945.0 & 30.3 & 30/30 & 1\textsuperscript{(*)}2\textsuperscript{(-)}3\textsuperscript{(-)} \\
 &  & 0.80 & 5876.2 & 80.4 & 26/30 & 2\textsuperscript{(-)}3\textsuperscript{(-)}4\textsuperscript{(*)} & \textbf{5978.7} & 9.0 & 30/30 & 1\textsuperscript{(+)}3\textsuperscript{(*)}4\textsuperscript{(+)} & 5974.0 & 13.3 & 30/30 & 1\textsuperscript{(+)}2\textsuperscript{(*)}4\textsuperscript{(+)} & 5945.0 & 30.3 & 30/30 & 1\textsuperscript{(*)}2\textsuperscript{(-)}3\textsuperscript{(-)} \\
 &  & 0.90 & 5879.2 & 80.5 & 25/30 & 2\textsuperscript{(-)}3\textsuperscript{(-)}4\textsuperscript{(*)} & 5975.3 & 11.4 & 30/30 & 1\textsuperscript{(+)}3\textsuperscript{(*)}4\textsuperscript{(+)} & \textbf{5975.7} & 10.4 & 30/30 & 1\textsuperscript{(+)}2\textsuperscript{(*)}4\textsuperscript{(+)} & 5945.0 & 30.3 & 30/30 & 1\textsuperscript{(*)}2\textsuperscript{(-)}3\textsuperscript{(-)} \\
\cmidrule(lr){2-19}
 & 0.15 & 0.60 & 5880.0 & 81.6 & 25/30 & 2\textsuperscript{(-)}3\textsuperscript{(-)}4\textsuperscript{(*)} & 5975.0 & 12.0 & 30/30 & 1\textsuperscript{(+)}3\textsuperscript{(*)}4\textsuperscript{(+)} & \textbf{5976.3} & 12.5 & 30/30 & 1\textsuperscript{(+)}2\textsuperscript{(*)}4\textsuperscript{(+)} & 5945.7 & 30.8 & 30/30 & 1\textsuperscript{(*)}2\textsuperscript{(-)}3\textsuperscript{(-)} \\
 &  & 0.70 & 5880.0 & 81.6 & 25/30 & 2\textsuperscript{(-)}3\textsuperscript{(-)}4\textsuperscript{(*)} & 5975.0 & 12.0 & 30/30 & 1\textsuperscript{(+)}3\textsuperscript{(*)}4\textsuperscript{(+)} & \textbf{5976.3} & 12.5 & 30/30 & 1\textsuperscript{(+)}2\textsuperscript{(*)}4\textsuperscript{(+)} & 5945.7 & 30.8 & 30/30 & 1\textsuperscript{(*)}2\textsuperscript{(-)}3\textsuperscript{(-)} \\
 &  & 0.80 & 5883.3 & 81.6 & 24/30 & 2\textsuperscript{(-)}3\textsuperscript{(-)}4\textsuperscript{(*)} & \textbf{5976.3} & 10.0 & 30/30 & 1\textsuperscript{(+)}3\textsuperscript{(*)}4\textsuperscript{(+)} & 5975.7 & 10.4 & 30/30 & 1\textsuperscript{(+)}2\textsuperscript{(*)}4\textsuperscript{(+)} & 5945.7 & 30.8 & 30/30 & 1\textsuperscript{(*)}2\textsuperscript{(-)}3\textsuperscript{(-)} \\
 &  & 0.90 & \textbf{5900.0} & 141.4 & 2/30 &  & - & - & 0/30 &  & - & - & 0/30 &  & - & - & 0/30 &  \\
\bottomrule
\end{tabular}
}
\end{table}

\section{Experimental Investigations}
\label{sec:experiments}

In this section, we describe the experimental setup and analyze the performance of the four algorithms under different uncertainty settings, uncertainty levels, and confidence levels.

\subsection{Experimental Settings}

We consider four benchmark instances (opm1, opm2, opm5, and opm6) derived from a real-world iron mining operation in Minas Gerais, Brazil~\cite{SOUZA20101041}. These instances have been used in previous studies of open-pit operational planning ~\cite{10.1145/3067695.3076004,alexandre2019multi, 10.1007/978-3-319-15892-1_29}. Each instance consists of 17 mining fronts, 8 shovels, and 30 trucks with heterogeneous capacities. The number of quality attributes is either 5 or 10, depending on the instance. Each quality attribute has lower and upper bounds that define the blending requirements for the processed ore. The instance characteristics are summarized in Table~\ref{tab:opm_instances}. For example, the pair $(15,50\,\text{t})$ denotes 15 trucks, each with a capacity of 50 tonnes. The planning horizon is fixed at $H=0.75$ hours.

We model the ore grades as independent normally distributed random variables, $\widetilde{q}_{oq}\sim\mathcal{N}(\mu_{oq},\sigma_{oq}^2)$. For each attribute $q$, we first compute the standard deviations $\kappa\mu_{oq}$ for all ore fronts and consider three assignments of these values to the fronts. S1 randomly assigns the values, with no systematic mean-uncertainty relationship. S2 uses $\sigma_{oq}=\kappa\mu_{oq}$, giving a positive relationship and S3 reverses the assignment, giving a negative relationship. We consider uncertainty scaling parameters $\kappa\in\{0.10,0.12,0.15\}$ and confidence levels $\alpha\in\{0.60,0.70,0.80,0.90\}$. For each confidence level, the quantile used in the deterministic chance-constraint conditions is $z_\beta=\Phi^{-1}\left((1+\alpha)/2\right)$.

We evaluate GSEMO, NSGA-II, NSGA-III, and MOEA/D. For NSGA-II, NSGA-III, and MOEA/D, the population size is $N=400$. GSEMO maintains a variable-size population and generates offspring using mutation only. For NSGA-II, NSGA-III, and MOEA/D, the crossover probability is $p_c=0.9$. For all algorithms, each trip-count and shovel-assignment variable is independently mutated with probability $1/n$, where $n$ is the total number of decision variables. The trip-relocation and truck-type replacement mutations are applied with probabilities $p_r=0.10$ and $p_t=0.05$, respectively. All algorithms are terminated after $5\times10^{5}$ fitness evaluations. For MOEA/D, the neighborhood size is set to $0.1N=40$, and at most $0.01N=4$ neighboring solutions can be replaced by each offspring. The same problem-specific mutation operator is used for all algorithms, while NSGA-II, NSGA-III, and MOEA/D use the same problem-specific crossover operator.

We evaluate each configuration over 30 independent runs using different random seeds. All experiments are conducted on single compute nodes of a high-performance computing cluster equipped with Intel\textsuperscript{\textregistered} Xeon\textsuperscript{\textregistered} Platinum 8360Y processors. Each run uses one CPU core and 500~MB of memory. All methods are implemented in Python~3.11.3 and executed in single-threaded mode.

For each run, we obtain the final non-dominated set and identify its feasible solutions. Among the feasible non-dominated solutions, we select the solution with the largest ore-production objective value $P(x)$. For each algorithm and experimental configuration, the mean and standard deviation of the best feasible $P(x)$ values are calculated over the runs that contain at least one feasible non-dominated solution. We also report the number of feasible runs out of the 30 independent runs. Runs without a feasible solution are excluded from the calculation of the mean and standard deviation. Therefore, the reported production values should be interpreted together with the corresponding number of feasible runs. We perform statistical comparisons using the Kruskal-Wallis test at a $95\%$ confidence level, followed by a Bonferroni post-hoc test for multiple comparisons~\cite{Corder_Foreman_2014}.

\begin{table*}[!t]
\centering
\caption{Mean and standard deviation of HV and IGD under S1 across all $\kappa$ values. 
}
\label{tab:hv_igd_setting_S1}
\resizebox{\textwidth}{!}{
\begin{tabular}{l c c rrrr rrrr rrrr rrrr}
\toprule
Instance & $\kappa$ & $\alpha$ & \multicolumn{4}{c}{GSEMO} & \multicolumn{4}{c}{NSGA-II} & \multicolumn{4}{c}{NSGA-III} & \multicolumn{4}{c}{MOEA/D} \\
\cmidrule(lr){4-7} \cmidrule(lr){8-11} \cmidrule(lr){12-15} \cmidrule(lr){16-19}
 & & & \multicolumn{2}{c}{HV} & \multicolumn{2}{c}{IGD} & \multicolumn{2}{c}{HV} & \multicolumn{2}{c}{IGD} & \multicolumn{2}{c}{HV} & \multicolumn{2}{c}{IGD} & \multicolumn{2}{c}{HV} & \multicolumn{2}{c}{IGD} \\
\cmidrule(lr){4-5} \cmidrule(lr){6-7} \cmidrule(lr){8-9} \cmidrule(lr){10-11} \cmidrule(lr){12-13} \cmidrule(lr){14-15} \cmidrule(lr){16-17} \cmidrule(lr){18-19}
 & & & Mean & Std & Mean & Std & Mean & Std & Mean & Std & Mean & Std & Mean & Std & Mean & Std & Mean & Std \\
\midrule
opm1 & 0.10 & 0.60 & 0.3585 & 0.3113 & 0.4509 & 0.2681 & \textbf{0.5382} & 0.0474 & \textbf{0.2784} & 0.0504 & 0.5277 & 0.0500 & 0.2853 & 0.0481 & 0.4112 & 0.0683 & 0.3688 & 0.0660 \\
 &  & 0.70 & 0.3585 & 0.3113 & 0.4509 & 0.2681 & \textbf{0.5382} & 0.0474 & \textbf{0.2784} & 0.0504 & 0.5277 & 0.0500 & 0.2853 & 0.0481 & 0.4112 & 0.0683 & 0.3688 & 0.0660 \\
 &  & 0.80 & 0.3839 & 0.3109 & 0.4285 & 0.2671 & \textbf{0.5382} & 0.0474 & \textbf{0.2784} & 0.0504 & 0.5277 & 0.0500 & 0.2853 & 0.0481 & 0.4112 & 0.0683 & 0.3688 & 0.0660 \\
 &  & 0.90 & 0.3903 & 0.3181 & 0.4225 & 0.2748 & \textbf{0.5392} & 0.0436 & \textbf{0.2724} & 0.0411 & 0.5265 & 0.0553 & 0.2891 & 0.0510 & 0.4118 & 0.0675 & 0.3688 & 0.0660 \\
\cmidrule(lr){2-19}
 & 0.12 & 0.60 & 0.3585 & 0.3113 & 0.4509 & 0.2681 & \textbf{0.5382} & 0.0474 & \textbf{0.2784} & 0.0504 & 0.5277 & 0.0500 & 0.2853 & 0.0481 & 0.4112 & 0.0683 & 0.3688 & 0.0660 \\
 &  & 0.70 & 0.3585 & 0.3113 & 0.4509 & 0.2681 & \textbf{0.5382} & 0.0474 & \textbf{0.2784} & 0.0504 & 0.5277 & 0.0500 & 0.2853 & 0.0481 & 0.4112 & 0.0683 & 0.3688 & 0.0660 \\
 &  & 0.80 & 0.3911 & 0.3192 & 0.4223 & 0.2751 & \textbf{0.5329} & 0.0488 & \textbf{0.2822} & 0.0478 & 0.5240 & 0.0519 & 0.2918 & 0.0501 & 0.4112 & 0.0683 & 0.3688 & 0.0660 \\
 &  & 0.90 &  &  &  &  &  &  &  &  &  &  &  &  &  &  &  &  \\
\cmidrule(lr){2-19}
 & 0.15 & 0.60 & 0.3706 & 0.3115 & 0.4407 & 0.2685 & \textbf{0.5382} & 0.0474 & \textbf{0.2784} & 0.0504 & 0.5277 & 0.0500 & 0.2853 & 0.0481 & 0.4112 & 0.0683 & 0.3688 & 0.0660 \\
 &  & 0.70 & 0.3911 & 0.3192 & 0.4223 & 0.2751 & \textbf{0.5380} & 0.0493 & \textbf{0.2779} & 0.0493 & 0.5268 & 0.0511 & 0.2936 & 0.0531 & 0.4118 & 0.0675 & 0.3688 & 0.0660 \\
 &  & 0.80 &  &  &  &  &  &  &  &  &  &  &  &  &  &  &  &  \\
 &  & 0.90 &  &  &  &  &  &  &  &  &  &  &  &  &  &  &  &  \\
\midrule
opm2 & 0.10 & 0.60 & 0.3702 & 0.3147 & 0.4743 & 0.2789 & \textbf{0.5135} & 0.0483 & \textbf{0.3187} & 0.0558 & 0.5088 & 0.0495 & 0.3214 & 0.0555 & 0.4076 & 0.0694 & 0.4035 & 0.0641 \\
 &  & 0.70 & 0.3702 & 0.3147 & 0.4743 & 0.2789 & 0.5116 & 0.0478 & 0.3191 & 0.0531 & \textbf{0.5125} & 0.0409 & \textbf{0.3147} & 0.0499 & 0.4076 & 0.0694 & 0.4035 & 0.0641 \\
 &  & 0.80 & 0.3751 & 0.3228 & 0.4820 & 0.2801 & 0.5024 & 0.0437 & 0.3402 & 0.0520 & \textbf{0.5038} & 0.0447 & \textbf{0.3389} & 0.0485 & 0.4076 & 0.0694 & 0.4162 & 0.0633 \\
 &  & 0.90 & 0.4493 & 0.2806 & 0.3980 & 0.2471 & \textbf{0.4721} & 0.0505 & \textbf{0.3586} & 0.0583 & 0.4685 & 0.0516 & 0.3631 & 0.0556 & 0.3677 & 0.0672 & 0.4344 & 0.0614 \\
\cmidrule(lr){2-19}
 & 0.12 & 0.60 & 0.3597 & 0.3158 & 0.4874 & 0.2775 & \textbf{0.5001} & 0.0541 & \textbf{0.3216} & 0.0543 & 0.4898 & 0.0435 & 0.3367 & 0.0479 & 0.3936 & 0.0597 & 0.4064 & 0.0564 \\
 &  & 0.70 & 0.3639 & 0.3219 & 0.4828 & 0.2832 & \textbf{0.5013} & 0.0482 & \textbf{0.3173} & 0.0468 & 0.4961 & 0.0468 & 0.3306 & 0.0485 & 0.3936 & 0.0597 & 0.4064 & 0.0564 \\
 &  & 0.80 & 0.3897 & 0.3002 & 0.4571 & 0.2646 & \textbf{0.5000} & 0.0443 & \textbf{0.3292} & 0.0542 & 0.4870 & 0.0476 & 0.3371 & 0.0527 & 0.3915 & 0.0596 & 0.4077 & 0.0564 \\
 &  & 0.90 &  &  &  &  &  &  &  &  &  &  &  &  &  &  &  &  \\
\cmidrule(lr){2-19}
 & 0.15 & 0.60 & 0.3664 & 0.3255 & 0.4792 & 0.2855 & \textbf{0.5144} & 0.0545 & \textbf{0.3140} & 0.0528 & 0.5001 & 0.0537 & 0.3215 & 0.0497 & 0.3890 & 0.0523 & 0.4176 & 0.0469 \\
 &  & 0.70 & 0.4139 & 0.3115 & 0.4351 & 0.2694 & 0.5041 & 0.0477 & 0.3198 & 0.0549 & \textbf{0.5055} & 0.0426 & \textbf{0.3145} & 0.0413 & 0.3913 & 0.0578 & 0.4155 & 0.0522 \\
 &  & 0.80 &  &  &  &  &  &  &  &  &  &  &  &  &  &  &  &  \\
 &  & 0.90 &  &  &  &  &  &  &  &  &  &  &  &  &  &  &  &  \\
\midrule
opm5 & 0.10 & 0.60 & 0.3736 & 0.3049 & 0.4313 & 0.2533 & \textbf{0.5408} & 0.0577 & \textbf{0.2792} & 0.0476 & 0.5284 & 0.0534 & 0.2942 & 0.0497 & 0.4005 & 0.0558 & 0.3689 & 0.0504 \\
 &  & 0.70 & 0.3736 & 0.3049 & 0.4313 & 0.2533 & \textbf{0.5408} & 0.0577 & \textbf{0.2792} & 0.0476 & 0.5284 & 0.0534 & 0.2942 & 0.0497 & 0.4005 & 0.0558 & 0.3689 & 0.0504 \\
 &  & 0.80 & 0.3736 & 0.3049 & 0.4313 & 0.2533 & \textbf{0.5408} & 0.0577 & \textbf{0.2792} & 0.0476 & 0.5284 & 0.0534 & 0.2942 & 0.0497 & 0.4005 & 0.0558 & 0.3689 & 0.0504 \\
 &  & 0.90 & 0.3736 & 0.3049 & 0.4313 & 0.2533 & \textbf{0.5408} & 0.0577 & \textbf{0.2792} & 0.0476 & 0.5284 & 0.0534 & 0.2942 & 0.0497 & 0.4005 & 0.0558 & 0.3689 & 0.0504 \\
\cmidrule(lr){2-19}
 & 0.12 & 0.60 & 0.3738 & 0.3052 & 0.4239 & 0.2556 & 0.5254 & 0.0532 & 0.2916 & 0.0483 & \textbf{0.5311} & 0.0539 & \textbf{0.2854} & 0.0502 & 0.4175 & 0.0705 & 0.3552 & 0.0565 \\
 &  & 0.70 & 0.3738 & 0.3052 & 0.4239 & 0.2556 & 0.5254 & 0.0532 & 0.2916 & 0.0483 & \textbf{0.5311} & 0.0539 & \textbf{0.2854} & 0.0502 & 0.4175 & 0.0705 & 0.3552 & 0.0565 \\
 &  & 0.80 & 0.3738 & 0.3052 & 0.4239 & 0.2556 & 0.5254 & 0.0532 & 0.2916 & 0.0483 & \textbf{0.5311} & 0.0539 & \textbf{0.2854} & 0.0502 & 0.4175 & 0.0705 & 0.3552 & 0.0565 \\
 &  & 0.90 & 0.3738 & 0.3052 & 0.4239 & 0.2556 & 0.5254 & 0.0532 & 0.2916 & 0.0483 & \textbf{0.5311} & 0.0539 & \textbf{0.2854} & 0.0502 & 0.4175 & 0.0705 & 0.3552 & 0.0565 \\
\cmidrule(lr){2-19}
 & 0.15 & 0.60 & 0.3573 & 0.3085 & 0.4436 & 0.2647 & \textbf{0.5390} & 0.0480 & \textbf{0.2611} & 0.0439 & 0.5338 & 0.0469 & 0.2679 & 0.0426 & 0.4157 & 0.0589 & 0.3515 & 0.0563 \\
 &  & 0.70 & 0.3573 & 0.3085 & 0.4436 & 0.2647 & \textbf{0.5390} & 0.0480 & \textbf{0.2611} & 0.0439 & 0.5338 & 0.0469 & 0.2679 & 0.0426 & 0.4157 & 0.0589 & 0.3515 & 0.0563 \\
 &  & 0.80 & 0.3573 & 0.3085 & 0.4436 & 0.2647 & \textbf{0.5390} & 0.0480 & \textbf{0.2611} & 0.0439 & 0.5338 & 0.0469 & 0.2679 & 0.0426 & 0.4157 & 0.0589 & 0.3515 & 0.0563 \\
 &  & 0.90 & 0.3849 & 0.3138 & 0.4426 & 0.2828 & \textbf{0.5386} & 0.0482 & \textbf{0.2519} & 0.0414 & 0.5294 & 0.0514 & 0.2588 & 0.0414 & 0.4157 & 0.0589 & 0.3469 & 0.0586 \\
\midrule
opm6 & 0.10 & 0.60 & 0.3969 & 0.2990 & 0.4472 & 0.2628 & \textbf{0.5234} & 0.0560 & \textbf{0.3101} & 0.0554 & 0.5144 & 0.0498 & 0.3235 & 0.0566 & 0.4114 & 0.0597 & 0.4122 & 0.0596 \\
 &  & 0.70 & 0.3969 & 0.2990 & 0.4472 & 0.2628 & \textbf{0.5234} & 0.0560 & \textbf{0.3101} & 0.0554 & 0.5144 & 0.0498 & 0.3235 & 0.0566 & 0.4114 & 0.0597 & 0.4122 & 0.0596 \\
 &  & 0.80 & 0.3969 & 0.2990 & 0.4472 & 0.2628 & \textbf{0.5234} & 0.0560 & \textbf{0.3101} & 0.0554 & 0.5144 & 0.0498 & 0.3235 & 0.0566 & 0.4114 & 0.0597 & 0.4122 & 0.0596 \\
 &  & 0.90 & 0.3969 & 0.2990 & 0.4472 & 0.2628 & \textbf{0.5234} & 0.0560 & \textbf{0.3101} & 0.0554 & 0.5144 & 0.0498 & 0.3235 & 0.0566 & 0.4114 & 0.0597 & 0.4122 & 0.0596 \\
\cmidrule(lr){2-19}
 & 0.12 & 0.60 & 0.3686 & 0.3109 & 0.4758 & 0.2688 & \textbf{0.5175} & 0.0480 & \textbf{0.3197} & 0.0500 & 0.5121 & 0.0572 & 0.3261 & 0.0553 & 0.4066 & 0.0679 & 0.4211 & 0.0591 \\
 &  & 0.70 & 0.3686 & 0.3109 & 0.4758 & 0.2688 & \textbf{0.5175} & 0.0480 & \textbf{0.3197} & 0.0500 & 0.5121 & 0.0572 & 0.3261 & 0.0553 & 0.4066 & 0.0679 & 0.4211 & 0.0591 \\
 &  & 0.80 & 0.3686 & 0.3109 & 0.4758 & 0.2688 & \textbf{0.5175} & 0.0480 & \textbf{0.3197} & 0.0500 & 0.5121 & 0.0572 & 0.3261 & 0.0553 & 0.4066 & 0.0679 & 0.4211 & 0.0591 \\
 &  & 0.90 & 0.3686 & 0.3109 & 0.4758 & 0.2688 & \textbf{0.5175} & 0.0480 & \textbf{0.3197} & 0.0500 & 0.5121 & 0.0572 & 0.3261 & 0.0553 & 0.4066 & 0.0679 & 0.4211 & 0.0591 \\
\cmidrule(lr){2-19}
 & 0.15 & 0.60 & 0.3659 & 0.3093 & 0.4770 & 0.2681 & 0.5196 & 0.0540 & 0.3265 & 0.0577 & \textbf{0.5212} & 0.0558 & \textbf{0.3235} & 0.0477 & 0.4269 & 0.0584 & 0.4076 & 0.0594 \\
 &  & 0.70 & 0.3659 & 0.3093 & 0.4770 & 0.2681 & 0.5196 & 0.0540 & 0.3265 & 0.0577 & \textbf{0.5212} & 0.0558 & \textbf{0.3235} & 0.0477 & 0.4269 & 0.0584 & 0.4076 & 0.0594 \\
 &  & 0.80 & 0.3659 & 0.3093 & 0.4770 & 0.2681 & 0.5196 & 0.0540 & 0.3265 & 0.0577 & \textbf{0.5212} & 0.0558 & \textbf{0.3235} & 0.0477 & 0.4269 & 0.0584 & 0.4076 & 0.0594 \\
 &  & 0.90 & 0.3633 & 0.3149 & 0.4817 & 0.2720 & 0.5223 & 0.0541 & \textbf{0.3175} & 0.0507 & \textbf{0.5229} & 0.0513 & 0.3282 & 0.0488 & 0.4269 & 0.0584 & 0.4076 & 0.0594 \\
\bottomrule
\end{tabular}
}
\end{table*}

We further evaluate the quality of the obtained feasible Pareto-front approximations using the hypervolume (HV) and inverted generational distance (IGD) indicators~\cite{Zitzler2003}. Since the true Pareto front is unknown, an empirical reference front is constructed separately for each combination of instance, uncertainty setting, $\kappa$, and $\alpha$. For each configuration, the feasible non-dominated solutions obtained from all independent runs and all compared algorithms are pooled, and the non-dominated subset of their union is used as the empirical reference front. If a saved non-dominated front contains both feasible and infeasible solutions, only its feasible solutions are used in the indicator calculation.

The objective values are normalized using the minimum and maximum objective values observed in the pooled feasible solution set for the corresponding configuration. Since the objectives are represented in minimization form for the indicator calculation, the ore-production objective is negated, while the fleet-cost objective remains unchanged. HV is computed in the normalized objective space using the reference point $(1.10,1.10)$, which is chosen slightly outside the normalized objective range. IGD is computed with respect to the normalized empirical reference front. Higher HV and lower IGD values indicate better Pareto-front approximation quality. The reported values correspond to the mean and standard deviation over runs containing at least one feasible non-dominated solution.

To provide additional information on constraint satisfaction, we report the final-population feasibility ratio and the deterministic chance-constraint margin. The feasibility ratio is the proportion of solutions in the final population that satisfy both the operational constraints and the deterministic chance-constraint conditions. The chance-constraint margin is the remaining slack in Eqs.~\eqref{eq:soc_lower}-\eqref{eq:soc_upper} and larger positive values indicate that a solution lies further from the corresponding chance-constraint boundary.

\subsection{Experimental Results}

This section presents the experimental results for the proposed approach using GSEMO, NSGA-II, NSGA-III, and MOEA/D under the three uncertainty settings S1-S3. The algorithms are compared in terms of feasibility, best feasible ore production, quality of the obtained non-dominated fronts, and computational time.

Figure~\ref{fig:feasibility_ratio} shows that the ability to obtain feasible solutions depends on both $\alpha$ and $\kappa$. For smaller values of these parameters, NSGA-II, NSGA-III, and MOEA/D generally maintain high feasibility ratios. As $\alpha$ and $\kappa$ increase, the feasibility ratio decreases for several instances.

This effect is particularly visible for opm1 and opm2. Under S1, for example, NSGA-II, NSGA-III, and MOEA/D obtain feasible solutions in all 30 runs for opm1 with $\kappa=0.12$ and $\alpha\leq0.80$, whereas no feasible solution is obtained at $\alpha=0.90$. For $\kappa=0.15$, no feasible solution is obtained for $\alpha\geq0.80$. A similar pattern is observed for opm2.

GSEMO generally obtains feasible solutions less consistently than NSGA-II, NSGA-III, and MOEA/D. This difference is also observed in several configurations where the fixed-population algorithms obtain feasible solutions in all runs. Overall, across most configurations in which feasible solutions are found, the fixed-population algorithms are more consistent than GSEMO in obtaining feasible solutions.

The effect of $\alpha$ and $\kappa$ also depends on the uncertainty setting. In particular, some configurations under S3 show a stronger reduction in feasibility than the corresponding configurations under S1 and S2. This indicates that feasibility depends not only on the overall uncertainty level, but also on how the standard deviations are assigned to the ore fronts.

Figure~\ref{fig:robustness_margin} provides additional information on constraint satisfaction by showing the average deterministic chance-constraint margin. Positive values indicate that the corresponding deterministic chance-constraint conditions are satisfied with remaining slack, while values closer to zero indicate solutions closer to the constraint boundary.

The chance-constraint margin varies with $\alpha$, $\kappa$, and the uncertainty setting. In general, increasing the confidence level makes the chance constraints more restrictive. The differences among S1-S3 also show that the assignment of uncertainty to the ore fronts affects the distance of the obtained solutions from the constraint boundary. These observations are consistent with the feasibility behavior shown in Figure~\ref{fig:feasibility_ratio}.

Tables~\ref{tab:f1_setting_S1}-\ref{tab:f1_setting_S3} report the best feasible ore-production value $P(x)$ for S1, S2, and S3, respectively, for $\kappa\in\{0.10,0.12,0.15\}$ and $\alpha\in\{0.60,0.70,0.80,0.90\}$. The mean and standard deviation are computed over feasible runs only, with the corresponding feasible-run count also reported. Bold values indicate the best mean performance for each configuration, and statistical comparisons are provided for the production results.

NSGA-II and NSGA-III generally obtain the highest best feasible ore production. MOEA/D usually obtains lower values than NSGA-II and NSGA-III, while GSEMO also produces lower mean values in many configurations and has fewer feasible runs in several cases. The statistical comparisons show significant differences between NSGA-II or NSGA-III and the other algorithms for many configurations.

The best feasible ore-production value does not always decrease immediately as $\alpha$ increases. In some configurations, similar production values are obtained for several consecutive confidence levels while the number of feasible runs decreases. This shows that increasing $\alpha$ can affect feasibility before producing a clear change in the best feasible production value.

Table~\ref{tab:hv_igd_setting_S1} reports the HV and IGD results for S1. The corresponding S2 and S3 results are provided in the Appendix~\ref{app_cc}. Higher HV and lower IGD indicate better Pareto-front approximation quality, and bold values indicate the best mean performance for each configuration.

NSGA-II and NSGA-III generally obtain the best HV and IGD values, while the differences between these two algorithms are often small. The algorithm with the best mean HV or IGD varies across individual configurations. GSEMO generally obtains lower HV and higher IGD values. MOEA/D performs better than GSEMO in several configurations, but its overall front quality is generally lower than that of NSGA-II and NSGA-III.

For some consecutive values of $\alpha$, identical or very similar HV and IGD values are observed. This indicates that increasing the confidence level does not necessarily lead to an immediate change in the feasible Pareto-front approximation.

\begin{table}[!t]
\centering
\caption{Mean and standard deviation of computing time in seconds across all experimental configurations and independent runs.}
\label{tab:runtime}
\resizebox{0.6\textwidth}{!}{%
\begin{tabular}{l rr rr rr rr}
\toprule
Instance & \multicolumn{2}{c}{GSEMO} & \multicolumn{2}{c}{NSGA-II} & \multicolumn{2}{c}{NSGA-III} & \multicolumn{2}{c}{MOEA/D} \\
\cmidrule(lr){2-3} \cmidrule(lr){4-5} \cmidrule(lr){6-7} \cmidrule(lr){8-9}
 & Mean & Std & Mean & Std & Mean & Std & Mean & Std \\
\midrule
opm1 & 274.4 & 28.3 & 682.2 & 146.7 & 744.2 & 183.1 & 626.7 & 66.3 \\
opm2 & 266.9 & 22.5 & 670.2 & 120.3 & 754.1 & 185.6 & 644.0 & 107.9 \\
opm5 & 266.2 & 54.5 & 606.7 & 58.2 & 675.0 & 113.1 & 643.5 & 141.9 \\
opm6 & 255.5 & 27.8 & 630.9 & 95.6 & 698.9 & 119.0 & 615.6 & 100.8 \\
\bottomrule
\end{tabular}%
}
\end{table}

Table~\ref{tab:runtime} reports the computational time of the four algorithms. GSEMO has the lowest computational cost, with mean runtimes of approximately 256-274 seconds across the four instances. NSGA-II requires approximately 607-682 seconds, NSGA-III approximately 675-754 seconds, and MOEA/D approximately 616-644 seconds.

Although GSEMO is considerably faster, it generally obtains lower solution quality and lower feasibility. NSGA-II and NSGA-III generally obtain higher best feasible ore-production values and better HV and IGD results, although at a higher computational cost. NSGA-II is also consistently faster than NSGA-III while providing similar solution quality.

Overall, NSGA-II and NSGA-III provide the best solution quality, with similar best feasible production, HV, and IGD results, while NSGA-II requires less computational time. MOEA/D obtains feasible solutions reliably for many configurations but generally provides lower solution quality. GSEMO has the lowest computational cost but shows lower feasibility and solution quality in many cases. The results further show that the effect of uncertainty depends on both its magnitude and its relationship with the expected ore grades.

\section{Conclusion}
\label{sec:conclusion}

In this paper, we investigated a chance-constrained bi-objective OPMOP under uncertain ore grades, simultaneously maximizing ore production and minimizing fleet cost. Under independent normally distributed ore grades, we derived sufficient deterministic conditions for the two-sided chance constraints, enabling evaluation without sampling or simulation.

We evaluated the proposed approach using GSEMO, NSGA-II, NSGA-III, and MOEA/D and showed that NSGA-II and NSGA-III generally provide the best solution quality, with NSGA-II requiring less computational time. GSEMO is considerably faster but generally shows lower feasibility and solution quality. The results further show that the effect of uncertainty depends on both its magnitude and its relationship with the expected ore grades.

Future work will consider additional sources of uncertainty, such as equipment availability and travel times, together with adaptive constraint-handling and search strategies for restrictive chance-constrained settings.

\bibliographystyle{unsrtnat}
\bibliography{lib}

\appendix
\section{Additional HV and IGD Results}
\label{app_cc}

This appendix provides the additional hypervolume (HV) and inverted generational distance (IGD) results for uncertainty settings S2 and S3. The corresponding results for S1 are presented in the main manuscript. The experimental settings and calculation of HV and IGD are described in the main paper.

\begin{table*}[htbp]
\centering
\caption{Hypervolume and IGD results under uncertainty setting S2 for all values of $\kappa$. Mean and standard deviation are computed over feasible recovered non-dominated fronts. HV and IGD use a configuration-specific empirical reference front and normalization. Bold values indicate the best mean performance for the corresponding metric and row.}
\label{tab:hv_igd_setting_S2}
\resizebox{\textwidth}{!}{%
\begin{tabular}{l c c rrrr rrrr rrrr rrrr}
\toprule
Instance & $\kappa$ & $\alpha$ & \multicolumn{4}{c}{GSEMO} & \multicolumn{4}{c}{NSGA-II} & \multicolumn{4}{c}{NSGA-III} & \multicolumn{4}{c}{MOEA/D} \\
\cmidrule(lr){4-7} \cmidrule(lr){8-11} \cmidrule(lr){12-15} \cmidrule(lr){16-19}
 & & & \multicolumn{2}{c}{HV} & \multicolumn{2}{c}{IGD} & \multicolumn{2}{c}{HV} & \multicolumn{2}{c}{IGD} & \multicolumn{2}{c}{HV} & \multicolumn{2}{c}{IGD} & \multicolumn{2}{c}{HV} & \multicolumn{2}{c}{IGD} \\
\cmidrule(lr){4-5} \cmidrule(lr){6-7} \cmidrule(lr){8-9} \cmidrule(lr){10-11} \cmidrule(lr){12-13} \cmidrule(lr){14-15} \cmidrule(lr){16-17} \cmidrule(lr){18-19}
 & & & Mean & Std & Mean & Std & Mean & Std & Mean & Std & Mean & Std & Mean & Std & Mean & Std & Mean & Std \\
\midrule
opm1 & 0.10 & 0.60 & 0.3828 & 0.3269 & 0.5033 & 0.2993 & \textbf{0.5552} & 0.0523 & \textbf{0.2761} & 0.0333 & 0.5525 & 0.0503 & 0.2771 & 0.0321 & 0.4413 & 0.0747 & 0.3576 & 0.0635 \\
 &  & 0.70 & 0.3828 & 0.3269 & 0.5033 & 0.2993 & \textbf{0.5552} & 0.0523 & \textbf{0.2761} & 0.0333 & 0.5525 & 0.0503 & 0.2771 & 0.0321 & 0.4413 & 0.0747 & 0.3576 & 0.0635 \\
 &  & 0.80 & 0.3828 & 0.3269 & 0.5033 & 0.2993 & \textbf{0.5552} & 0.0523 & \textbf{0.2761} & 0.0333 & 0.5525 & 0.0503 & 0.2771 & 0.0321 & 0.4413 & 0.0747 & 0.3576 & 0.0635 \\
 &  & 0.90 & 0.4176 & 0.3413 & 0.5329 & 0.3094 & 0.5520 & 0.0431 & 0.3384 & 0.0297 & \textbf{0.5569} & 0.0397 & \textbf{0.3348} & 0.0306 & 0.4409 & 0.0744 & 0.4188 & 0.0653 \\
\cmidrule(lr){2-19}
 & 0.12 & 0.60 & 0.3828 & 0.3269 & 0.5033 & 0.2993 & \textbf{0.5552} & 0.0523 & \textbf{0.2761} & 0.0333 & 0.5525 & 0.0503 & 0.2771 & 0.0321 & 0.4413 & 0.0747 & 0.3576 & 0.0635 \\
 &  & 0.70 & 0.3828 & 0.3269 & 0.5033 & 0.2993 & \textbf{0.5552} & 0.0523 & \textbf{0.2761} & 0.0333 & 0.5525 & 0.0503 & 0.2771 & 0.0321 & 0.4413 & 0.0747 & 0.3576 & 0.0635 \\
 &  & 0.80 & 0.3824 & 0.3267 & 0.4970 & 0.2931 & \textbf{0.5593} & 0.0509 & \textbf{0.2848} & 0.0363 & 0.5557 & 0.0499 & 0.2876 & 0.0342 & 0.4409 & 0.0744 & 0.3655 & 0.0611 \\
 &  & 0.90 &  &  &  &  &  &  &  &  &  &  &  &  &  &  &  &  \\
\cmidrule(lr){2-19}
 & 0.15 & 0.60 & 0.3828 & 0.3269 & 0.5033 & 0.2993 & \textbf{0.5552} & 0.0523 & \textbf{0.2761} & 0.0333 & 0.5525 & 0.0503 & 0.2771 & 0.0321 & 0.4413 & 0.0747 & 0.3576 & 0.0635 \\
 &  & 0.70 & 0.3824 & 0.3267 & 0.4970 & 0.2931 & \textbf{0.5602} & 0.0518 & \textbf{0.2842} & 0.0372 & 0.5593 & 0.0516 & 0.2844 & 0.0328 & 0.4409 & 0.0744 & 0.3655 & 0.0611 \\
 &  & 0.80 &  &  &  &  &  &  &  &  &  &  &  &  &  &  &  &  \\
 &  & 0.90 &  &  &  &  &  &  &  &  &  &  &  &  &  &  &  &  \\
\midrule
opm2 & 0.10 & 0.60 & 0.3238 & 0.3040 & 0.4995 & 0.2651 & \textbf{0.4662} & 0.0454 & 0.3580 & 0.0377 & 0.4582 & 0.0435 & \textbf{0.3578} & 0.0350 & 0.3886 & 0.0435 & 0.4066 & 0.0388 \\
 &  & 0.70 & 0.3227 & 0.3021 & 0.4745 & 0.2663 & \textbf{0.4601} & 0.0436 & 0.3523 & 0.0370 & 0.4599 & 0.0487 & \textbf{0.3512} & 0.0424 & 0.3858 & 0.0490 & 0.3971 & 0.0436 \\
 &  & 0.80 & 0.3221 & 0.3025 & 0.4634 & 0.2601 & \textbf{0.4675} & 0.0418 & \textbf{0.3572} & 0.0433 & 0.4573 & 0.0484 & 0.3668 & 0.0426 & 0.3870 & 0.0487 & 0.4112 & 0.0435 \\
 &  & 0.90 & 0.3405 & 0.3084 & 0.4867 & 0.2719 & \textbf{0.4634} & 0.0432 & 0.3802 & 0.0393 & 0.4583 & 0.0433 & \textbf{0.3786} & 0.0396 & 0.3791 & 0.0515 & 0.4277 & 0.0453 \\
\cmidrule(lr){2-19}
 & 0.12 & 0.60 & 0.3332 & 0.3037 & 0.4653 & 0.2660 & 0.4546 & 0.0409 & 0.3515 & 0.0385 & \textbf{0.4586} & 0.0434 & \textbf{0.3457} & 0.0402 & 0.3926 & 0.0527 & 0.3957 & 0.0451 \\
 &  & 0.70 & 0.3342 & 0.3045 & 0.4523 & 0.2615 & 0.4638 & 0.0436 & 0.3677 & 0.0446 & \textbf{0.4658} & 0.0453 & \textbf{0.3667} & 0.0459 & 0.3904 & 0.0533 & 0.4082 & 0.0493 \\
 &  & 0.80 & 0.3579 & 0.3204 & 0.4271 & 0.2484 & \textbf{0.5632} & 0.0461 & \textbf{0.3149} & 0.0462 & 0.5528 & 0.0526 & 0.3182 & 0.0457 & 0.4721 & 0.0614 & 0.3620 & 0.0451 \\
 &  & 0.90 &  &  &  &  &  &  &  &  &  &  &  &  &  &  &  &  \\
\cmidrule(lr){2-19}
 & 0.15 & 0.60 & 0.3322 & 0.3036 & 0.4579 & 0.2610 & \textbf{0.4655} & 0.0429 & 0.3614 & 0.0429 & 0.4642 & 0.0517 & \textbf{0.3579} & 0.0485 & 0.3914 & 0.0398 & 0.4099 & 0.0442 \\
 &  & 0.70 & 0.3318 & 0.2939 & 0.4561 & 0.2660 & 0.4674 & 0.0469 & 0.3078 & 0.0399 & \textbf{0.4685} & 0.0425 & \textbf{0.3077} & 0.0342 & 0.3942 & 0.0472 & 0.3569 & 0.0356 \\
 &  & 0.80 &  &  &  &  &  &  &  &  &  &  &  &  &  &  &  &  \\
 &  & 0.90 &  &  &  &  &  &  &  &  &  &  &  &  &  &  &  &  \\
\midrule
opm5 & 0.10 & 0.60 & 0.3828 & 0.3269 & 0.5033 & 0.2993 & \textbf{0.5552} & 0.0523 & \textbf{0.2761} & 0.0333 & 0.5525 & 0.0503 & 0.2771 & 0.0321 & 0.4413 & 0.0747 & 0.3576 & 0.0635 \\
 &  & 0.70 & 0.3828 & 0.3269 & 0.5033 & 0.2993 & \textbf{0.5552} & 0.0523 & \textbf{0.2761} & 0.0333 & 0.5525 & 0.0503 & 0.2771 & 0.0321 & 0.4413 & 0.0747 & 0.3576 & 0.0635 \\
 &  & 0.80 & 0.3828 & 0.3269 & 0.5033 & 0.2993 & \textbf{0.5552} & 0.0523 & \textbf{0.2761} & 0.0333 & 0.5525 & 0.0503 & 0.2771 & 0.0321 & 0.4413 & 0.0747 & 0.3576 & 0.0635 \\
 &  & 0.90 & 0.3828 & 0.3269 & 0.5033 & 0.2993 & \textbf{0.5552} & 0.0523 & \textbf{0.2761} & 0.0333 & 0.5525 & 0.0503 & 0.2771 & 0.0321 & 0.4413 & 0.0747 & 0.3576 & 0.0635 \\
\cmidrule(lr){2-19}
 & 0.12 & 0.60 & 0.3828 & 0.3269 & 0.5033 & 0.2993 & \textbf{0.5552} & 0.0523 & \textbf{0.2761} & 0.0333 & 0.5525 & 0.0503 & 0.2771 & 0.0321 & 0.4413 & 0.0747 & 0.3576 & 0.0635 \\
 &  & 0.70 & 0.3828 & 0.3269 & 0.5033 & 0.2993 & \textbf{0.5552} & 0.0523 & \textbf{0.2761} & 0.0333 & 0.5525 & 0.0503 & 0.2771 & 0.0321 & 0.4413 & 0.0747 & 0.3576 & 0.0635 \\
 &  & 0.80 & 0.3828 & 0.3269 & 0.5033 & 0.2993 & \textbf{0.5552} & 0.0523 & \textbf{0.2761} & 0.0333 & 0.5525 & 0.0503 & 0.2771 & 0.0321 & 0.4413 & 0.0747 & 0.3576 & 0.0635 \\
 &  & 0.90 & 0.3828 & 0.3269 & 0.5033 & 0.2993 & \textbf{0.5552} & 0.0523 & \textbf{0.2761} & 0.0333 & 0.5525 & 0.0503 & 0.2771 & 0.0321 & 0.4413 & 0.0747 & 0.3576 & 0.0635 \\
\cmidrule(lr){2-19}
 & 0.15 & 0.60 & 0.3828 & 0.3269 & 0.5033 & 0.2993 & \textbf{0.5552} & 0.0523 & \textbf{0.2761} & 0.0333 & 0.5525 & 0.0503 & 0.2771 & 0.0321 & 0.4413 & 0.0747 & 0.3576 & 0.0635 \\
 &  & 0.70 & 0.3828 & 0.3269 & 0.5033 & 0.2993 & \textbf{0.5552} & 0.0523 & \textbf{0.2761} & 0.0333 & 0.5525 & 0.0503 & 0.2771 & 0.0321 & 0.4413 & 0.0747 & 0.3576 & 0.0635 \\
 &  & 0.80 & 0.3828 & 0.3269 & 0.5033 & 0.2993 & \textbf{0.5552} & 0.0523 & \textbf{0.2761} & 0.0333 & 0.5525 & 0.0503 & 0.2771 & 0.0321 & 0.4413 & 0.0747 & 0.3576 & 0.0635 \\
 &  & 0.90 & 0.3828 & 0.3269 & 0.5033 & 0.2993 & \textbf{0.5522} & 0.0533 & \textbf{0.2787} & 0.0352 & 0.5519 & 0.0522 & 0.2787 & 0.0342 & 0.4413 & 0.0747 & 0.3576 & 0.0635 \\
\midrule
opm6 & 0.10 & 0.60 & 0.3399 & 0.3198 & 0.5053 & 0.2754 & \textbf{0.5014} & 0.0511 & \textbf{0.3443} & 0.0449 & 0.4824 & 0.0462 & 0.3597 & 0.0405 & 0.3975 & 0.0515 & 0.4152 & 0.0421 \\
 &  & 0.70 & 0.3399 & 0.3198 & 0.5053 & 0.2754 & \textbf{0.5014} & 0.0511 & \textbf{0.3443} & 0.0449 & 0.4824 & 0.0462 & 0.3597 & 0.0405 & 0.3975 & 0.0515 & 0.4152 & 0.0421 \\
 &  & 0.80 & 0.3399 & 0.3198 & 0.5053 & 0.2754 & \textbf{0.4972} & 0.0487 & \textbf{0.3486} & 0.0430 & 0.4868 & 0.0435 & 0.3553 & 0.0397 & 0.3975 & 0.0515 & 0.4152 & 0.0421 \\
 &  & 0.90 & 0.3350 & 0.3277 & 0.5103 & 0.2788 & \textbf{0.4858} & 0.0503 & 0.3573 & 0.0448 & 0.4833 & 0.0526 & \textbf{0.3559} & 0.0463 & 0.3949 & 0.0527 & 0.4192 & 0.0426 \\
\cmidrule(lr){2-19}
 & 0.12 & 0.60 & 0.3439 & 0.3224 & 0.5040 & 0.2833 & \textbf{0.4906} & 0.0492 & \textbf{0.3483} & 0.0402 & 0.4857 & 0.0438 & 0.3568 & 0.0391 & 0.3929 & 0.0605 & 0.4174 & 0.0499 \\
 &  & 0.70 & 0.3439 & 0.3224 & 0.5040 & 0.2833 & \textbf{0.4912} & 0.0498 & \textbf{0.3479} & 0.0405 & 0.4863 & 0.0460 & 0.3567 & 0.0404 & 0.3929 & 0.0605 & 0.4174 & 0.0499 \\
 &  & 0.80 & 0.3419 & 0.3195 & 0.5048 & 0.2822 & \textbf{0.4965} & 0.0377 & \textbf{0.3446} & 0.0329 & 0.4823 & 0.0521 & 0.3591 & 0.0430 & 0.3929 & 0.0605 & 0.4174 & 0.0499 \\
 &  & 0.90 & 0.3541 & 0.3499 & 0.5063 & 0.3100 & \textbf{0.3694} & 0.0891 & \textbf{0.4049} & 0.0881 & 0.3523 & 0.0961 & 0.4160 & 0.0915 & 0.2738 & 0.0860 & 0.4940 & 0.0928 \\
\cmidrule(lr){2-19}
 & 0.15 & 0.60 & 0.3528 & 0.3374 & 0.5530 & 0.3035 & 0.4903 & 0.0485 & 0.3428 & 0.0377 & \textbf{0.4923} & 0.0442 & \textbf{0.3414} & 0.0335 & 0.3866 & 0.0686 & 0.4249 & 0.0658 \\
 &  & 0.70 & 0.3391 & 0.3167 & 0.5067 & 0.2802 & \textbf{0.4981} & 0.0465 & \textbf{0.3433} & 0.0372 & 0.4888 & 0.0535 & 0.3559 & 0.0448 & 0.3930 & 0.0594 & 0.4207 & 0.0461 \\
 &  & 0.80 & 0.3716 & 0.3831 & 0.5456 & 0.3226 & \textbf{0.4203} & 0.0765 & \textbf{0.4387} & 0.0655 & 0.4151 & 0.0721 & 0.4412 & 0.0607 & 0.3073 & 0.1006 & 0.5359 & 0.0904 \\
 &  & 0.90 & 0.2536 & 0.2488 & 0.4504 & 0.2516 & \textbf{0.4386} & 0.0885 & \textbf{0.2716} & 0.0717 & 0.4225 & 0.0823 & 0.2817 & 0.0647 & 0.3295 & 0.1078 & 0.3577 & 0.0848 \\
\bottomrule
\end{tabular}%
}
\end{table*}

\begin{table*}[!t]
\centering
\caption{Hypervolume and IGD results under uncertainty setting S3 for all values of $\kappa$. Mean and standard deviation are computed over feasible recovered non-dominated fronts. HV and IGD use a configuration-specific empirical reference front and normalization. Bold values indicate the best mean performance for the corresponding metric and row.}
\label{tab:hv_igd_setting_S3}
\resizebox{\textwidth}{!}{%
\begin{tabular}{l c c rrrr rrrr rrrr rrrr}
\toprule
Instance & $\kappa$ & $\alpha$ & \multicolumn{4}{c}{GSEMO} & \multicolumn{4}{c}{NSGA-II} & \multicolumn{4}{c}{NSGA-III} & \multicolumn{4}{c}{MOEA/D} \\
\cmidrule(lr){4-7} \cmidrule(lr){8-11} \cmidrule(lr){12-15} \cmidrule(lr){16-19}
 & & & \multicolumn{2}{c}{HV} & \multicolumn{2}{c}{IGD} & \multicolumn{2}{c}{HV} & \multicolumn{2}{c}{IGD} & \multicolumn{2}{c}{HV} & \multicolumn{2}{c}{IGD} & \multicolumn{2}{c}{HV} & \multicolumn{2}{c}{IGD} \\
\cmidrule(lr){4-5} \cmidrule(lr){6-7} \cmidrule(lr){8-9} \cmidrule(lr){10-11} \cmidrule(lr){12-13} \cmidrule(lr){14-15} \cmidrule(lr){16-17} \cmidrule(lr){18-19}
 & & & Mean & Std & Mean & Std & Mean & Std & Mean & Std & Mean & Std & Mean & Std & Mean & Std & Mean & Std \\
\midrule
opm1 & 0.10 & 0.60 & 0.3739 & 0.3074 & 0.4394 & 0.2757 & \textbf{0.5260} & 0.0515 & \textbf{0.2836} & 0.0419 & 0.5201 & 0.0587 & 0.2849 & 0.0462 & 0.4134 & 0.0643 & 0.3621 & 0.0503 \\
 &  & 0.70 & 0.3739 & 0.3074 & 0.4394 & 0.2757 & \textbf{0.5260} & 0.0515 & \textbf{0.2836} & 0.0419 & 0.5201 & 0.0587 & 0.2849 & 0.0462 & 0.4134 & 0.0643 & 0.3621 & 0.0503 \\
 &  & 0.80 & 0.3739 & 0.3074 & 0.4394 & 0.2757 & \textbf{0.5260} & 0.0515 & \textbf{0.2836} & 0.0419 & 0.5201 & 0.0587 & 0.2849 & 0.0462 & 0.4134 & 0.0643 & 0.3621 & 0.0503 \\
 &  & 0.90 & 0.3845 & 0.2946 & 0.4185 & 0.2682 & \textbf{0.5382} & 0.0523 & \textbf{0.2521} & 0.0456 & 0.5205 & 0.0511 & 0.2687 & 0.0480 & 0.4134 & 0.0643 & 0.3520 & 0.0552 \\
\cmidrule(lr){2-19}
 & 0.12 & 0.60 & 0.3739 & 0.3074 & 0.4394 & 0.2757 & \textbf{0.5260} & 0.0515 & \textbf{0.2836} & 0.0419 & 0.5201 & 0.0587 & 0.2849 & 0.0462 & 0.4134 & 0.0643 & 0.3621 & 0.0503 \\
 &  & 0.70 & 0.3739 & 0.3074 & 0.4394 & 0.2757 & \textbf{0.5260} & 0.0515 & \textbf{0.2836} & 0.0419 & 0.5201 & 0.0587 & 0.2849 & 0.0462 & 0.4134 & 0.0643 & 0.3621 & 0.0503 \\
 &  & 0.80 & 0.3763 & 0.3064 & 0.4090 & 0.2701 & \textbf{0.5294} & 0.0601 & \textbf{0.2740} & 0.0514 & 0.5162 & 0.0532 & 0.2868 & 0.0449 & 0.4134 & 0.0643 & 0.3566 & 0.0454 \\
 &  & 0.90 &  &  &  &  &  &  &  &  &  &  &  &  &  &  &  &  \\
\cmidrule(lr){2-19}
 & 0.15 & 0.60 & 0.3739 & 0.3074 & 0.4394 & 0.2757 & \textbf{0.5260} & 0.0515 & \textbf{0.2836} & 0.0419 & 0.5201 & 0.0587 & 0.2849 & 0.0462 & 0.4134 & 0.0643 & 0.3621 & 0.0503 \\
 &  & 0.70 & 0.3763 & 0.3064 & 0.4090 & 0.2701 & 0.5246 & 0.0563 & 0.2782 & 0.0495 & \textbf{0.5268} & 0.0573 & \textbf{0.2779} & 0.0466 & 0.4134 & 0.0643 & 0.3566 & 0.0454 \\
 &  & 0.80 &  &  &  &  &  &  &  &  &  &  &  &  &  &  &  &  \\
 &  & 0.90 &  &  &  &  &  &  &  &  &  &  &  &  &  &  &  &  \\
\midrule
opm2 & 0.10 & 0.60 & 0.4057 & 0.3231 & 0.4606 & 0.2835 & \textbf{0.5338} & 0.0508 & \textbf{0.2992} & 0.0418 & 0.5232 & 0.0463 & 0.3082 & 0.0432 & 0.4262 & 0.0599 & 0.3713 & 0.0556 \\
 &  & 0.70 & 0.4057 & 0.3231 & 0.4606 & 0.2835 & \textbf{0.5338} & 0.0508 & \textbf{0.2992} & 0.0418 & 0.5232 & 0.0463 & 0.3082 & 0.0432 & 0.4262 & 0.0599 & 0.3713 & 0.0556 \\
 &  & 0.80 & 0.4054 & 0.3234 & 0.4614 & 0.2845 & \textbf{0.5346} & 0.0503 & \textbf{0.2966} & 0.0403 & 0.5244 & 0.0457 & 0.3088 & 0.0433 & 0.4351 & 0.0555 & 0.3633 & 0.0520 \\
 &  & 0.90 & \textbf{0.4754} & 0.3872 & \textbf{0.3405} & 0.3255 &  &  &  &  &  &  &  &  &  &  &  &  \\
\cmidrule(lr){2-19}
 & 0.12 & 0.60 & 0.4201 & 0.3143 & 0.4451 & 0.2760 & \textbf{0.5340} & 0.0429 & \textbf{0.2972} & 0.0357 & 0.5320 & 0.0462 & 0.3002 & 0.0407 & 0.4219 & 0.0541 & 0.3751 & 0.0485 \\
 &  & 0.70 & 0.4201 & 0.3143 & 0.4451 & 0.2760 & 0.5328 & 0.0425 & 0.2982 & 0.0366 & \textbf{0.5343} & 0.0450 & \textbf{0.2959} & 0.0376 & 0.4219 & 0.0541 & 0.3751 & 0.0485 \\
 &  & 0.80 & \textbf{0.3493} & 0.3308 & \textbf{0.5215} & 0.2780 & 0.0733 & 0.0635 & 0.8115 & 0.1115 & 0.0917 & 0.0778 & 0.7793 & 0.1365 & 0.0367 & 0.0000 & 0.8758 & 0.0000 \\
 &  & 0.90 &  &  &  &  &  &  &  &  &  &  &  &  &  &  &  &  \\
\cmidrule(lr){2-19}
 & 0.15 & 0.60 & 0.3880 & 0.3015 & 0.4390 & 0.2686 & \textbf{0.5122} & 0.0416 & \textbf{0.2938} & 0.0418 & 0.5067 & 0.0508 & 0.2978 & 0.0466 & 0.4053 & 0.0473 & 0.3726 & 0.0497 \\
 &  & 0.70 & \textbf{0.3619} & 0.3553 & \textbf{0.4654} & 0.3269 &  &  &  &  &  &  &  &  &  &  &  &  \\
 &  & 0.80 &  &  &  &  &  &  &  &  &  &  &  &  &  &  &  &  \\
 &  & 0.90 &  &  &  &  &  &  &  &  &  &  &  &  &  &  &  &  \\
\midrule
opm5 & 0.10 & 0.60 & 0.3739 & 0.3074 & 0.4394 & 0.2757 & \textbf{0.5260} & 0.0515 & \textbf{0.2836} & 0.0419 & 0.5201 & 0.0587 & 0.2849 & 0.0462 & 0.4134 & 0.0643 & 0.3621 & 0.0503 \\
 &  & 0.70 & 0.3739 & 0.3074 & 0.4394 & 0.2757 & \textbf{0.5260} & 0.0515 & \textbf{0.2836} & 0.0419 & 0.5201 & 0.0587 & 0.2849 & 0.0462 & 0.4134 & 0.0643 & 0.3621 & 0.0503 \\
 &  & 0.80 & 0.3739 & 0.3074 & 0.4394 & 0.2757 & \textbf{0.5260} & 0.0515 & \textbf{0.2836} & 0.0419 & 0.5201 & 0.0587 & 0.2849 & 0.0462 & 0.4134 & 0.0643 & 0.3621 & 0.0503 \\
 &  & 0.90 & 0.3739 & 0.3074 & 0.4394 & 0.2757 & \textbf{0.5260} & 0.0515 & \textbf{0.2836} & 0.0419 & 0.5201 & 0.0587 & 0.2849 & 0.0462 & 0.4134 & 0.0643 & 0.3621 & 0.0503 \\
\cmidrule(lr){2-19}
 & 0.12 & 0.60 & 0.3739 & 0.3074 & 0.4394 & 0.2757 & \textbf{0.5260} & 0.0515 & \textbf{0.2836} & 0.0419 & 0.5201 & 0.0587 & 0.2849 & 0.0462 & 0.4134 & 0.0643 & 0.3621 & 0.0503 \\
 &  & 0.70 & 0.3739 & 0.3074 & 0.4394 & 0.2757 & \textbf{0.5260} & 0.0515 & \textbf{0.2836} & 0.0419 & 0.5201 & 0.0587 & 0.2849 & 0.0462 & 0.4134 & 0.0643 & 0.3621 & 0.0503 \\
 &  & 0.80 & 0.3739 & 0.3074 & 0.4394 & 0.2757 & \textbf{0.5260} & 0.0515 & \textbf{0.2836} & 0.0419 & 0.5201 & 0.0587 & 0.2849 & 0.0462 & 0.4134 & 0.0643 & 0.3621 & 0.0503 \\
 &  & 0.90 & 0.3739 & 0.3074 & 0.4394 & 0.2757 & \textbf{0.5260} & 0.0515 & \textbf{0.2836} & 0.0419 & 0.5201 & 0.0587 & 0.2849 & 0.0462 & 0.4134 & 0.0643 & 0.3621 & 0.0503 \\
\cmidrule(lr){2-19}
 & 0.15 & 0.60 & 0.3739 & 0.3074 & 0.4394 & 0.2757 & \textbf{0.5260} & 0.0515 & \textbf{0.2836} & 0.0419 & 0.5201 & 0.0587 & 0.2849 & 0.0462 & 0.4134 & 0.0643 & 0.3621 & 0.0503 \\
 &  & 0.70 & 0.3739 & 0.3074 & 0.4394 & 0.2757 & \textbf{0.5260} & 0.0515 & \textbf{0.2836} & 0.0419 & 0.5201 & 0.0587 & 0.2849 & 0.0462 & 0.4134 & 0.0643 & 0.3621 & 0.0503 \\
 &  & 0.80 & 0.3739 & 0.3074 & 0.4394 & 0.2757 & \textbf{0.5260} & 0.0515 & \textbf{0.2836} & 0.0419 & 0.5201 & 0.0587 & 0.2849 & 0.0462 & 0.4134 & 0.0643 & 0.3621 & 0.0503 \\
 &  & 0.90 & 0.3763 & 0.3064 & 0.4090 & 0.2701 & \textbf{0.5268} & 0.0531 & \textbf{0.2761} & 0.0431 & 0.5170 & 0.0532 & 0.2873 & 0.0466 & 0.4134 & 0.0643 & 0.3566 & 0.0454 \\
\midrule
opm6 & 0.10 & 0.60 & 0.3881 & 0.3342 & 0.5087 & 0.2923 & \textbf{0.4812} & 0.0474 & \textbf{0.3849} & 0.0369 & 0.4719 & 0.0518 & 0.3925 & 0.0439 & 0.3871 & 0.0568 & 0.4499 & 0.0508 \\
 &  & 0.70 & 0.3881 & 0.3342 & 0.5087 & 0.2923 & \textbf{0.4812} & 0.0474 & \textbf{0.3849} & 0.0369 & 0.4719 & 0.0518 & 0.3925 & 0.0439 & 0.3871 & 0.0568 & 0.4499 & 0.0508 \\
 &  & 0.80 & 0.3881 & 0.3342 & 0.5087 & 0.2923 & \textbf{0.4812} & 0.0474 & \textbf{0.3849} & 0.0369 & 0.4719 & 0.0518 & 0.3925 & 0.0439 & 0.3871 & 0.0568 & 0.4499 & 0.0508 \\
 &  & 0.90 & 0.3881 & 0.3342 & 0.5087 & 0.2923 & \textbf{0.4824} & 0.0470 & \textbf{0.3837} & 0.0364 & 0.4730 & 0.0496 & 0.3909 & 0.0417 & 0.3871 & 0.0568 & 0.4499 & 0.0508 \\
\cmidrule(lr){2-19}
 & 0.12 & 0.60 & 0.3668 & 0.3184 & 0.4671 & 0.2691 & \textbf{0.5252} & 0.0480 & \textbf{0.3016} & 0.0421 & 0.5117 & 0.0589 & 0.3097 & 0.0476 & 0.4294 & 0.0758 & 0.3723 & 0.0603 \\
 &  & 0.70 & 0.3668 & 0.3184 & 0.4671 & 0.2691 & \textbf{0.5252} & 0.0480 & \textbf{0.3016} & 0.0421 & 0.5117 & 0.0589 & 0.3097 & 0.0476 & 0.4294 & 0.0758 & 0.3723 & 0.0603 \\
 &  & 0.80 & 0.3668 & 0.3184 & 0.4671 & 0.2691 & \textbf{0.5252} & 0.0480 & \textbf{0.3016} & 0.0421 & 0.5117 & 0.0589 & 0.3097 & 0.0476 & 0.4294 & 0.0758 & 0.3723 & 0.0603 \\
 &  & 0.90 & 0.3794 & 0.3184 & 0.4558 & 0.2682 & \textbf{0.5183} & 0.0413 & 0.3083 & 0.0410 & 0.5153 & 0.0467 & \textbf{0.3076} & 0.0399 & 0.4260 & 0.0733 & 0.3768 & 0.0583 \\
\cmidrule(lr){2-19}
 & 0.15 & 0.60 & 0.3816 & 0.3233 & 0.4549 & 0.2719 & 0.5197 & 0.0484 & 0.3102 & 0.0442 & \textbf{0.5235} & 0.0534 & \textbf{0.3046} & 0.0457 & 0.4193 & 0.0658 & 0.3836 & 0.0555 \\
 &  & 0.70 & 0.3816 & 0.3233 & 0.4549 & 0.2719 & 0.5170 & 0.0488 & 0.3100 & 0.0422 & \textbf{0.5293} & 0.0521 & \textbf{0.2992} & 0.0454 & 0.4193 & 0.0658 & 0.3836 & 0.0555 \\
 &  & 0.80 & 0.3953 & 0.3228 & 0.4425 & 0.2705 & 0.5193 & 0.0505 & 0.3095 & 0.0439 & \textbf{0.5287} & 0.0552 & \textbf{0.3012} & 0.0500 & 0.4193 & 0.0658 & 0.3836 & 0.0555 \\
 &  & 0.90 & \textbf{0.4767} & 0.6128 & \textbf{0.4191} & 0.5927 &  &  &  &  &  &  &  &  &  &  &  &  \\
\bottomrule
\end{tabular}%
}
\end{table*}

\end{document}